\documentclass{article}

\usepackage{amsthm,amsfonts,amsmath,amssymb}
\usepackage{thmtools}
\usepackage{float}
\usepackage{color}
\usepackage{graphicx} 

\usepackage{enumitem}		
\usepackage{fontawesome5}   
\usepackage{subcaption}     
\usepackage{mathtools}

\renewcommand{\Pr}{\mathop{\bf Pr\/}}

\newcommand{\opt}{\textnormal{OPT}}

\def\<{\langle}
\def\>{\rangle}

\newcommand{\Vendi}{\textnormal{Vendi}}
\usepackage{microtype}
\usepackage{graphicx}
\usepackage{subcaption}
\usepackage{booktabs} 

\usepackage{hyperref}

\usepackage[accepted]{icml2026}

\usepackage{amsmath}
\usepackage{amssymb}
\usepackage{mathtools}
\usepackage{amsthm}

\usepackage[capitalize,noabbrev]{cleveref}

\theoremstyle{plain}
\newtheorem{theorem}{Theorem}[section]

\newtheorem{lemma}[theorem]{Lemma}

\theoremstyle{definition}

\theoremstyle{remark}

\usepackage[textsize=tiny]{todonotes}

\icmltitlerunning{Retriever Portfolios: A Principled Approach to Adaptive RAG}

\usepackage{tikz}
\usepackage{subcaption} 
\usetikzlibrary{shapes.geometric, arrows.meta, decorations.pathreplacing, positioning, fit, backgrounds}

\usepackage[most]{tcolorbox}
\tcbuselibrary{listings,breakable}
\usepackage{xcolor}

\newtcblisting{promptbox}[1]{
  title={#1},
  listing only,
  breakable,
  colback=white,
  colframe=black!25,
  left=2mm,right=2mm,top=1mm,bottom=1mm,
  listing options={
    basicstyle=\ttfamily\footnotesize,
    breaklines=true,
    columns=fullflexible
  }
}

\begin{document}

\twocolumn[
  \icmltitle{Retriever Portfolios: A Principled Approach to Adaptive RAG}



  \icmlsetsymbol{equal}{*}

  \begin{icmlauthorlist}
    \icmlauthor{Miltiadis Stouras}{epfl}
    \icmlauthor{Vincent Cohen-Addad}{google}
    \icmlauthor{Silvio Lattanzi}{google}
    \icmlauthor{Ola Svensson}{epfl}
  \end{icmlauthorlist}

  \icmlaffiliation{epfl}{EPFL}
  \icmlaffiliation{google}{Google Research}

  \icmlcorrespondingauthor{Miltiadis Stouras}{miltiadis.stouras@epfl.ch}

  \icmlkeywords{retrieval augmented generation, large language models, retrieval, QA, inference-time tuning, diverse retrievers, portfolios}

  \vskip 0.3in
]



\printAffiliationsAndNotice{}  

\begin{abstract}
Retrieval-augmented generation (RAG) systems typically rely on a single
retriever and a single set of hyperparameters, despite facing highly
heterogeneous queries that range from simple factoid questions to complex
multi-hop reasoning.
We propose a method that automatically selects a small, diverse subset of retrievers (a portfolio) from a large pool of candidates, to cover different regions of the target query distribution.
We formalize this setting via
an expected best-of-$k$ objective over the query distribution and show that it admits an efficient portfolio construction algorithm with near-optimal guarantees. Across multiple QA benchmarks, our learned
portfolios and router pipeline consistently outperform
single-retriever and naive multi-retriever baselines on both retrieval metrics and answer quality. In addition, compared to inference-time hyperparameter tuning approaches, fixed portfolios enable parallel retrieval and LLM calls, achieving comparable (and sometimes better) accuracy with substantially lower latency and token cost.

\end{abstract}

\section{Introduction}

Retrieval-augmented generation (RAG) has emerged as a standard paradigm for
grounding large language models (LLMs) in external knowledge sources. By
retrieving documents from a corpus and conditioning generation on the retrieved
evidence, RAG can improve factual accuracy and reduce hallucinations across a
range of tasks, including open-domain QA and knowledge-intensive dialogue
\citep{lewis2020rag,izacard2021leveraging,shuster2021retrieval}. In this setup, a
retriever maps a user query to a ranked list of documents, and the LLM produces
an answer conditioned on both the query and the retrieved context. Despite
substantial progress on retrieval models and architectures, the current practice in
RAG systems is overwhelmingly to select a \emph{single} retriever (and a single
set of hyperparameters)  and deploy it uniformly for all
queries.

However, the query distributions faced by RAG systems are highly heterogeneous.
Some questions are simple factoid queries that can be answered from a single,
lexically similar passage, while others require multi-hop reasoning over
multiple documents, aggregation of semantically diverse evidence, or
domain-specific terminology. Therefore, different retrieval configurations tend to
specialize in different regimes 
\citep{karpukhin2020dense,jeong2024adaptive}. Even for a fixed retriever,
recent work shows that it can be beneficial to \emph{adapt} its hyperparameters 
at inference time to better suit each query
\citep{rezaei2025vendi}. 
As a result, there is growing evidence that no single
retriever is optimal for all queries, and that fixing a single retrieval strategy
can leave substantial performance on the table across diverse information needs
\citep{kalra2025mor}.


To address this limitation several recent works explicitly aim to adapt their retrieval strategy to the query.
Adaptive-RAG \citep{jeong2024adaptive} trains a classifier to predict the
``complexity level'' of the given query and then chooses among a small,
fixed menu of strategies (no retrieval, single-step retrieval, or iterative
multi-hop retrieval) based on this prediction. Vendi-RAG
\citep{rezaei2025vendi} instead focuses on a parametric retriever and
iteratively tunes its hyperparameter at inference time: for each query,
it generates candidate answers under different settings and employs an
LLM judge to steer the retriever towards a better relevance--diversity
trade-off. 
While these approaches validate the benefits of adapting retrieval, they either
operate over a very small, hand-designed set of strategies or perform expensive
per-query search in a hyperparameter space. 

In contrast, we propose
a principled approach for selecting a small
portfolio of retrievers with complementary behavior, from a large candidate pool, 
that is provably near-optimal for the underlying query distribution.

We tackle this problem by explicitly modeling the fact that different
retrievers may be best for different questions. We view each query
$q$ as a sample from an underlying distribution, and each retriever $r$
(including its architecture and hyperparameters) as a candidate tool
for answering that query. For every query-retriever pair $(q,r)$, we
define a score $s(q,r)$ that measures how suitable $r$ is for $q$,
for example in terms of retrieval quality or end-to-end answer
correctness. Rather than searching for a single retriever that performs
best on average, we instead choose a small set $S$ of $k$ retrievers
and evaluate this set by asking: for a random query $q$, how good is
the \emph{best} retriever in $S$? Formally, we measure the quality of
$S$ by the expected ``best-of-$k$'' score
$\mathbb{E}_{q}[\max_{r \in S} s(q,r)]$ over the query distribution.
This objective encourages portfolios whose members specialize on
different regions of the query space, since redundant retrievers that
behave similarly contribute little to the maximum. In the rest of the
paper we analyze this objective and design algorithms that, given a
finite sample of queries and a large candidate pool of retrievers,
efficiently construct small portfolios with provable guarantees and
strong empirical performance.


\paragraph{Our contributions} We introduce the new problem, RAG portfolio
optimization, and analyze it theoretically and practically. More formally:
\begin{itemize}
    \item\textbf{Formulation \& Theory.} Given a large pool of candidate retrievers, we
  formalize an \emph{expected best-of-$k$} objective over the query distribution
  that measures how well a portfolio of $k$ retrievers covers heterogeneous
  information needs. Furthermore, we prove that one can efficiently obtain
  a near-optimal portfolio of retrievers.

    \item\textbf{Methods.} We propose a pipeline that learns a static portfolio of 
    complementary retrievers offline (Algorithm~\ref{alg:efficient_greedy})
    and trains a lightweight router to dynamically 
    select the best retriever per query (Figure~\ref{fig:inference_pipeline}). 
    This approach amortizes the cost of adaptation, 
    avoiding expensive iterative search at inference time.

    \item\textbf{Empirical evaluation.} We evaluate our approach on diverse open-domain 
    and multi-hop QA benchmarks (HotpotQA, 2WikiMultihopQA, TriviaQA, and MusiQue). Our 
    method consistently yields better retrieval recall and answer accuracy compared to 
    single-retriever baselines and inference-time tuning methods like Vendi-RAG, while 
    significantly reducing latency and token usage.

\end{itemize}

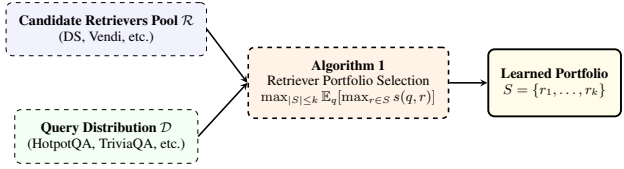
\begin{figure}[t]
\centering
\resizebox{\columnwidth}{!}{
\begin{tikzpicture}[
    font=\small, 
    >=stealth,
    block/.style={draw, rounded corners, align=center, inner sep=8pt, fill=white},
    dashed_block/.style={block, dashed}, 
    pool_style/.style={dashed_block, fill=blue!5},
    data_style/.style={dashed_block, fill=green!5},
    algo_style/.style={dashed_block, line width=1pt, fill=orange!10},
    port_style/.style={block, line width=1.2pt, fill=yellow!10}
]

    \node[pool_style] (pool) at (-1.2, 1.2) {
        \textbf{Candidate Retrievers Pool $\mathcal{R}$} \\
        (DS, Vendi, etc.)
    };
    
    \node[data_style] (data) at (-1.2, -1.2) {
        \textbf{Query Distribution $\mathcal{D}$} \\ 
        (HotpotQA, TriviaQA, etc.)
    };

    \node[algo_style] (algo) at (4.2, 0) {
        \textbf{Algorithm 1} \\
        Retriever Portfolio Selection \\
        $\max_{|S|\le k} \mathbb{E}_{q}[\max_{r \in S} s(q,r)]$
    };

    \node[port_style, align=center, minimum height=1.5cm] (portfolio) at (8.8, 0) {
        \textbf{Learned Portfolio} \\
        $S = \{r_1, \dots, r_k\}$
    };

    \draw[->, line width=1pt] (pool.east) -- (algo.west);
    \draw[->, line width=1pt] (data.east) -- (algo.west);
    \draw[->, line width=1pt] (algo) -- (portfolio);

\end{tikzpicture}
}
\caption{Offline Portfolio Optimization: Selecting a diverse subset of size $k$ from a large pool of size $m$ to cover the query distribution.}
\label{fig:offline_opt}
\end{figure}
\begin{figure*}[t]
\centering
\definecolor{answercolor}{HTML}{E8F7F0}

\begin{tikzpicture}[
    font=\small, 
    scale=0.97,
    transform shape,
    >=stealth,
    block/.style={draw, rounded corners, align=center, inner sep=5pt, fill=white},
    container_style/.style={draw, dashed, rounded corners, inner sep=8pt, fill=blue!2},
    query_node/.style={block, dashed, fill=red!20, font=\footnotesize},
    llm_style/.style={block, thick, fill=orange!20},
    router_block/.style={block, fill=blue!5},
    vector_style/.style={draw, fill=purple!10, minimum width=0.4cm, minimum height=1.0cm},
    score_style/.style={draw, thin, inner sep=2pt, font=\tiny, fill=blue!10},
    tick_style/.style={text=green!60!black, font=\large\bfseries},
    cross_style/.style={text=red!70!black, font=\large\bfseries},
    answer_style/.style={draw, solid, fill=answercolor, align=center, inner sep=3pt}
]

    \node[query_node] (query) at (-0.25, 0) {Query text\\$\mathbf{q}$};

    \node[router_block] (encoder) at (2.25, 0) {Encoder\\(T5+MPNet/E5)};
    \node[vector_style] (enc_q) at (4.45, 0) {$Enc(\mathbf{q})$};
    
    \begin{scope}[on background layer]
        \node[container_style, fit=(encoder) (enc_q)] (router_box) {};
    \end{scope}
    \node[anchor=south, font=\footnotesize\bfseries] at (router_box.north) {Router};

    \node[score_style] (s1) at (5.78, 1.2) {$s_1$};
    \node[score_style] (s2) at (5.78, 0.6) {$s_2$};
    \node[score_style] (s3) at (5.78, 0.0) {$s_3$};
    \node[score_style] (sk) at (5.78, -1.0) {$s_k$};
    \node at (5.78, -0.45) {$\vdots$};
    \node[anchor=south, font=\tiny\itshape, align=center] at (5.78, 1.5) {Similarities with retrievers'\\learned embeddings};

    \draw[->, thick] (6.27, 0) -- node[above, font=\footnotesize] {top-$\ell$} node[below, font=\footnotesize] {filtering} (6.97, 0);
    
    \node[tick_style] (tick1) at (7.3, 1.2) {\checkmark};
    \node[cross_style] (cross1) at (7.3, 0.6) {$\times$};
    \node[tick_style] (tick2) at (7.3, 0.0) {\checkmark};
    \node at (7.3, -0.45) {$\vdots$}; 
    \node[cross_style] (cross2) at (7.3, -1.0) {$\times$};

    \draw[->, thick] (7.6, 0) -- (8.0, 0);

    \node[query_node] (pq1) at (8.6, 0.7) {$\mathbf{q}$};
    \node[block] (r1) at (9.6, 0.7) {$r_1$};
    \node[llm_style] (ans1) at (11.4, 0.7) {Answer LLM};

    \node[query_node] (pql) at (8.6, -0.7) {$\mathbf{q}$};
    \node[block] (rl) at (9.6, -0.7) {$r_{\ell}$};
    \node[llm_style] (ansl) at (11.4, -0.7) {Answer LLM};
    
    \node at (11.0, 0) {$\vdots$};

    \draw[->] (pq1) -- (r1); \draw[->] (r1) -- (ans1);
    \draw[->] (pql) -- (rl); \draw[->] (rl) -- (ansl);

    \begin{scope}[on background layer]
        \node[container_style, fit=(pq1) (ans1) (ansl)] (par_box) {};
    \end{scope}
    \node[anchor=south, font=\footnotesize\bfseries] at (par_box.north) {Parallel Retrieval \& Answering};

    \node[answer_style, font=\tiny] (a1) at (13.2, 0.7) {$a_1$};
    \node at (13.2, 0) {$\vdots$}; 
    \node[answer_style, font=\tiny] (al) at (13.2, -0.7) {$a_\ell$};
    \node[llm_style] (selector) at (14.5, 0) {Selector\\ LLM};
    
    \node[answer_style, font=\footnotesize] (final) at (16.0, 0) {Final\\ answer};

    \draw[->, thick] (query) -- (encoder);
    \draw[->] (encoder) -- (enc_q);
    
    \draw[->, thin] (enc_q.east) -- (s1.west);
    \draw[->, thin] (enc_q.east) -- (s2.west);
    \draw[->, thin] (enc_q.east) -- (s3.west);
    \draw[->, thin] (enc_q.east) -- (sk.west);
    
    \draw[->] (ans1.east) -- (a1.west);
    \draw[->] (ansl.east) -- (al.west);
    \draw[->] (a1.east) -- (selector.west);
    \draw[->] (al.east) -- (selector.west);
    \draw[->, thick] (selector) -- (final);

\end{tikzpicture}
\caption{Adaptive Inference Pipeline: The router encodes query $\mathbf{q}$ and identifies best portfolio members. Selected components execute in parallel to minimize latency, followed by selector aggregation.}
\label{fig:inference_pipeline}
\end{figure*}
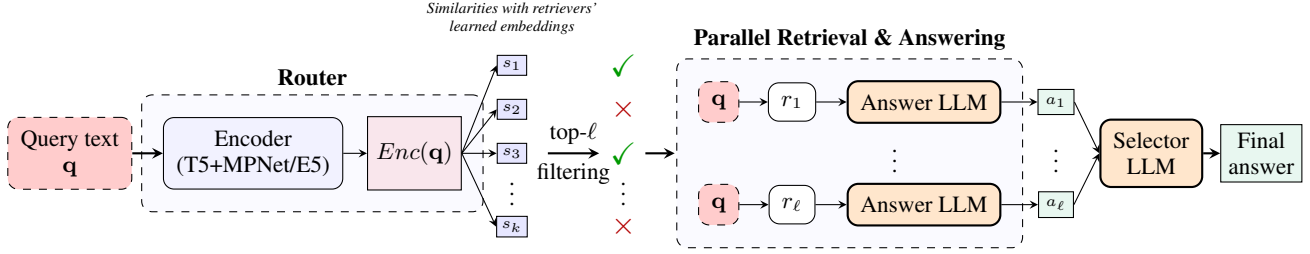

\paragraph{Related Work} There are three main areas that are closely
related to our work: retrieval augmented generation, adaptive retrieval approaches (including mixture
of retrievers and routing pipelines) and portfolio optimization.

\vspace{0.7em}
\emph{Retrieval Augmented Generation.}
Retrieval-augmented generation (RAG) has become a standard approach for
grounding large language models (LLMs) in external knowledge. Instead of
relying solely on parametric memory, a RAG system retrieves a small set of
relevant documents from a large corpus and conditions generation on both the
query and the retrieved context, which improves factual accuracy and
knowledge coverage on open-domain and knowledge-intensive tasks
\citep{lewis2020rag,izacard2021leveraging,shuster2021retrieval}. Early work
combined neural retrievers with sequence-to-sequence generators for
open-domain QA~\citep{karpukhin2020dense,lewis2020rag}, and subsequent work has
extended this paradigm to more complex settings, including multi-hop
reasoning and conversational agents~\citep{shuster2021retrieval}. Across
these systems, the retriever is a key component: it is usually chosen once,
tuned on a development set, and then fixed for all queries, despite the
diversity of retrieval behaviors that different questions may require.

\vspace{0.7em}
\emph{Adaptive Retrieval.} Recent work challenges the "one-size-fits-all" 
retrieval paradigm. Approaches like Adaptive-RAG~\cite{jeong2024adaptive} classify 
query complexity to select among fixed strategies (e.g., no retrieval vs. multi-step retrieval). 
Vendi-RAG~\cite{rezaei2025vendi} adopts a more granular approach, iteratively 
tuning a diversity-relevance hyperparameter at inference time to optimize the 
retrieved set. Other methods, such as Self-RAG~\cite{asai2024self}, introduce 
retrieval tokens to dynamically decide when to retrieve or critique retrieved 
passages. Unlike these methods, which often incur high inference latency through 
iterative generation or search, our approach shifts the optimization burden 
offline by learning a fixed portfolio of complementary retrievers that can be routed to efficiently.

\vspace{0.7em}
\emph{Mixtures of Retrievers and Query Routing.} 
Recent work exploits complementary retrieval signals either by \emph{mixing} retrievers
or \emph{routing} queries to different experts.
MoR~\cite{mor} proposes a zero-shot \emph{mixture of retrievers} that assigns per-query
weights to heterogeneous retrievers using pre- and post-retrieval signals, then fuses
their scores to produce a single ranked list.
While MoR validates the benefits of leveraging multiple retrieval signals, it focuses on
score-level fusion over a given set of retrievers.
In contrast, our goal is to \emph{select} a small retriever \emph{portfolio} from a large candidate
pool with provable near-optimal coverage under an expected best-of-$k$ objective, and to
route each query to only a small number of portfolio members, enabling predictable and
parallelizable inference cost when downstream LLM calls are expensive.

On the routing side, RouterDC~\cite{routerdc} studies query-based routing for assembling
multiple expert LLMs via dual contrastive learning, explicitly handling cases where multiple
experts perform well for the same query.
Our setting is analogous, we route queries to retrievers rather than generators, but our main
contribution lies in the \emph{portfolio construction problem} and its guarantees; RouterDC-style
router training is complementary and can be incorporated into our pipeline.

\vspace{0.7em}
\emph{Algorithm Selection and Solution Portfolios.}
Arising from catalog optimization and data mining, portfolio optimization has 
been studied extensively in algorithmic literature.
In the Segmentation Problems, initially introduced by Kleinberg, Papadimitriou and 
Raghavan~\cite{kleinberg2004segmentation}, for any combinatorial optimization problem and 
a set $S$ of different cost vectors for this problem, the segmentation problem asks to 
partition the set $S$ into several \textit{segments} and pick a separate solution for each 
segment, so that the total cost is minimized. In an alternative version of this problem 
Gupta, Moondra and Singh~\cite{gupta2025balancing} studied a variant where the objective 
is to identify a small set of solutions that guarantees a good approximation for each 
one of the cost vectors of interest. In portfolio optimization~\cite{portfolios-itcs}, one is given a combinatorial 
problem, a set of value functions  over the solutions of the problem, and a 
distribution  over the value, and the goal is to select $k$ solutions  
that maximize or minimize the expected value of the {\em best} of those solutions.
Our work instantiates this theoretical framework for RAG, treating retrievers as 
candidate algorithms and queries as problem instances.

\section{Problem Formulation}
\label{sec:problem}

We formalize the retrieval process in RAG as a portfolio selection problem. 
At a high level, we are given a large pool of candidate retrievers and a distribution over queries. Our goal is to select a small subset (a \emph{portfolio}) such that, in expectation, performance is high when, for each query, we are allowed to use the best-performing retriever from the portfolio.

\subsection{Setting and Notation}
\label{subsec:setting}

Let $\mathcal{Q}$ denote the set of questions, and let
$\mathcal{D}$ be an unknown distribution over $\mathcal{Q}$. Let also $\mathcal{R} = \{ r_1, r_2, \dots, r_m \}$ be the set of candidate retrievers.
Each retriever $r \in \mathcal{R}$ is a black-box procedure that, given a query
$q \in \mathcal{Q}$, returns a ranked list of documents. 
We do not assume any particular retrieval algorithm or parameterization; different
retrievers may correspond to different models, strategies, or hyperparameters.

We equip the space of query--retriever pairs with a bounded
\emph{score function}
\[
  s : \mathcal{Q} \times \mathcal{R} \to [0, 1],
\]
where $s(q, r)$ quantifies how suitable retriever $r$ is for answering query
$q$ within a given RAG pipeline.
The definition of
$s(q, r)$ is intentionally abstract here; in later sections we instantiate it
with retrieval metrics and end-to-end answer quality.

\subsection{Portfolio Objective}
\label{subsec:portfolio-objective}

A \emph{portfolio} is a subset $S \subseteq \mathcal{R}$ of retrievers. We fix
a budget $k \in \mathbb{N}$ and restrict our attention to portfolios with
$|S| \le k$. Intuitively, $k$ captures the maximum number of distinct retrievers
one is willing to maintain and potentially deploy in a system.

Given a portfolio $S$ and a query $q$, we define the score of the portfolio on
$q$ as the score of its best retriever:
\[
  \mathrm{score}(q, S)
  \;=\;
  \max_{r \in S} s(q, r).
\]
In words, a portfolio performs well on a query if at least one of its members
does.

The quality of a portfolio $S$ with respect to the query distribution
$\mathcal{D}$ is then measured by the following objective
\begin{equation*}
  F(S)
  \;=\;
  \mathbb{E}_{q \sim \mathcal{D}}
  \big[ \max_{r \in S} s(q, r) \big]
  \;=\;
  \mathbb{E}_{q \sim \mathcal{D}}
  \big[ \mathrm{score}(q, S) \big].
\end{equation*}
The function $F(S)$ is the expected score of the portfolio on a random query.
A portfolio is good when, for most queries sampled from $\mathcal{D}$, at least
one retriever in $S$ performs well; in other words, $F(S)$ captures how well
$S$ covers the heterogeneous behaviors required across the query distribution.

Our goal is to find a portfolio of size at most $k$ that maximizes $F$:
\begin{equation*}
  \max_{S \subseteq \mathcal{R} : |S| \le k} F(S).
\end{equation*}
This is exactly a data-driven solution portfolio problem \cite{portfolios-itcs}, and closely related to
the catalogue problem~\cite{kleinberg2004segmentation}, with retrievers
playing the role of candidate solutions and queries acting as problem
instances.

\section{Selecting Retriever Portfolios from Data}\label{sec:theory}
In this section we present an algorithm (\cref{alg:efficient_greedy}) for
computing a near-optimal portfolio, given access to a distribution of queries.

Recall  from Section~\ref{sec:problem} that, 
given a candidate pool of retrievers $\mathcal R$ and a query distribution $\mathcal D$,
we seek a portfolio $S\subseteq \mathcal R$ of size at most $k$ maximizing the
distributional \emph{best-of-$k$} objective
\begin{equation}
\label{eq:pop_objective}
F(S)\;:=\;\mathbb E_{q\sim\mathcal D}\Big[\max_{r\in S} s(q,r)\Big],
\end{equation}
where $s(q,r)\in[0,1]$ denotes the  quality of retriever $r$ on query $q$.

Algorithm~\ref{alg:efficient_greedy} implements a sample-based greedy procedure.
It first draws $N$ i.i.d.\ queries $Q=\{q_1,\dots,q_N\}\sim\mathcal D^N$ and replaces
the population objective~\eqref{eq:pop_objective} with the empirical estimate
\begin{equation}
\label{eq:emp_objective}
\widehat F_Q(S)\;:=\;\frac1N\sum_{q\in Q}\max_{r\in S} s(q,r).
\end{equation}
The algorithm then constructs $S$ iteratively: at each step it chooses the retriever
$r$ that maximizes the empirical \emph{marginal gain}
$\widehat F_Q(S\cup\{r\})-\widehat F_Q(S)$ and adds it to the portfolio.

To compute marginal gains efficiently, the algorithm maintains a memoized vector
$V[q]=\max_{r'\in S}s(q,r')$ over sampled queries. The marginal gain of adding $r$
can then be written as
\[
\widehat F_Q(S\cup\{r\})-\widehat F_Q(S)
=\frac1N\sum_{q\in Q}\max\!\bigl(0,\;s(q,r)-V[q]\bigr),
\]
so we only need to compare each candidate against the current best score per query.

\paragraph{Running time.}
Overall, the implementation runs in $\mathcal O(k\,|\mathcal R|\,N) = \mathcal O\left(k\,|\mathcal R|\cdot \frac{k \log |\mathcal{R}| + \log(1/\delta)}{\epsilon^2}\right)$ score evaluations
(given access to $s(q,r)$) and $\mathcal O(N)$ additional memory. It is important to note that this is independent of the support/complexity of the query distribution which may be very large. 

\paragraph{Approximation guarantee.}
The function $S\mapsto \max_{r\in S}s(q,r)$ is non-negative, monotone and submodular for every fixed
$q$, and therefore both $F(\cdot)$ and $\widehat F_Q(\cdot)$ are non-negative monotone submodular
set functions (this was similarly observed in the theoretical studies \cite{kleinberg2004segmentation,portfolios-itcs})\footnote{A set function $f:2^V\to\mathbb{R}$ is \emph{non-negative} if $f(S)\geq 0$. 
It is \emph{monotone} if $f(A)\le f(B)$ for all $A\subseteq B\subseteq V$. 
It is \emph{submodular} if it satisfies the diminishing returns property, i.e., for all 
$A\subseteq B\subseteq V$ and $x\in V\setminus B$,
\(
f(A\cup\{x\})-f(A)\ge f(B\cup\{x\})-f(B).
\)
}.
As a consequence, the classical greedy algorithm achieves a $(1-1/e)$-approximation
to the \emph{empirical} optimum under a cardinality constraint \cite{nemhauser1978mp}.

It remains to relate the empirical objective~\eqref{eq:emp_objective} to the population
objective~\eqref{eq:pop_objective}. There are at most
$\binom{|\mathcal R|}{k}$ portfolios of size $k$ (and at most
$\sum_{j=0}^k\binom{|\mathcal R|}{j}$ of size $\le k$). By Hoeffding/Chernoff bounds
and a union bound over this family, it suffices to take
\begin{equation}
\label{eq:sample_complexity}
N=\mathcal O\!\left(\frac{k\log|\mathcal R|+\log(1/\delta)}{\epsilon^2}\right)
\end{equation}
samples to ensure uniform concentration of $\widehat F_Q(S)$ around $F(S)$ for all
$|S|\le k$ with probability at least $1-\delta$.

Putting these pieces  together (see \cref{app:greedy_portfolio_proofs} for detailed arguments) yields the following near-optimality guarantee (where we assume score evaluations take time $\mathcal O(1)$):

\begin{theorem}
    Algorithm~\ref{alg:efficient_greedy} runs in time $\mathcal O\left(k\,|\mathcal R|\cdot \frac{k \log |\mathcal{R}| + \log(1/\delta)}{\epsilon^2}\right)$ and returns a solution $S$ such that with probability at least $1-\delta$,
\(
F(S) \geq (1-\frac1e)\opt-\epsilon,
\)
where  $\opt = \max_{|T|\le k}F(T)$.
\label{thm:greedy-analysis}
\end{theorem}
In words, Algorithm~\ref{alg:efficient_greedy} attains the approximation guarantee of $(1-1/e)$ up to an additional $\epsilon$ statistical error term. We remark that the $(1-1/e)$ guarantee is tight, as with no further assumptions, our model can express any coverage function:  it is well-known that, for coverage functions, the $1-1/e$ guarantee is tight for polynomial-time algorithms both in the value-oracle model~\cite{wolsey1982mor} and if $\mathrm{P} \neq \mathrm{NP}$~\cite{feige1998jacm}. 

\begin{algorithm}[tb]
   \caption{Greedy Portfolio Selection}
   \label{alg:efficient_greedy}
\begin{algorithmic}
   \STATE {\bfseries Input:} Candidate pool $\mathcal{R}$, distribution $\mathcal{D}$, budget $k$, error $\epsilon$, failure probability $\delta$ \\[2mm]
   
   \STATE {\bfseries 1. Sample Complexity:}
   \STATE Sample $N$ queries $Q = \{q_1, \dots, q_N\}$ independently from $\mathcal{D}$, where $N = \mathcal{O}\left(\frac{k \log |\mathcal{R}| + \log(1/\delta)}{\epsilon^2}\right)$\\[2mm]
   
   \STATE {\bfseries 2. Initialization:}
   \STATE $S \leftarrow \emptyset$
   \STATE $V[q] \leftarrow 0$ for all $q \in Q$ \\[2mm]
   
   \STATE {\bfseries 3. Greedy Selection:}
   \FOR{$i=1$ {\bfseries to} $k$}
      \STATE $r^* \leftarrow \arg\max_{r \in \mathcal{R} \setminus S} \sum_{q \in Q} \max(0, s(q, r) - V[q])$
      \STATE $S \leftarrow S \cup \{r^*\}$
      \FORALL{$q \in Q$}
         \STATE $V[q] \leftarrow \max(V[q], s(q, r^*))$ 
      \ENDFOR
   \ENDFOR\\[2mm]
   \STATE {\bfseries Output:} Near-optimal portfolio $S$
\end{algorithmic}
\end{algorithm}


\section{Pipeline \& Experimental Setup}
\label{sec:experimental-setup}

We evaluate retriever portfolios on four open-domain QA benchmarks and report both retrieval-only coverage metrics and end-to-end answer accuracy.
All experiments are run locally on NVIDIA A100 80GB GPUs.
Additional implementation details (exact hyperparameter grids, retriever pseudocode, router training settings, and prompt templates) are provided in \cref{sec:app-implementation}.

\subsection{Routing and inference-time selection}
\label{sec:exp:routing}

At inference time, a learned \emph{contrastive} router represents the query $\mathbf{q}$ using its raw text together with cached MPNet and E5 query embeddings.
A frozen Flan-T5-Large encoder \citep{flant5} produces a contextual text representation, which is fused with the two backbone-specific dense embeddings (see \cref{app:router} for more details).
The router scores each of the $k$ portfolio retrievers by similarity to a learned retriever embedding and executes the top-$\ell$ of the portfolio. This pipeline is visualized in Figure~\ref{fig:inference_pipeline}.
The router is trained from retrieval supervision with a multi-positive contrastive objective, following the approach of \citet{routerdc}, where positives are the retrievers that achieve the best Recall@$k$ for each training query. 
The router training details (loss, tie handling, optimization schedule, and hyperparameters) are deferred to \cref{app:router}.

\subsection{Answer and judge models}
\label{sec:exp:models}

We use open-weight instruction-tuned LLMs as answerers in the RAG pipeline and as selectors for selecting among candidate answers.
In our main experiments, we report end-to-end results with \textsc{Gemma-3-27B-It} \citep{gemma3} and \textsc{Llama-3.1-70B-Instruct} \citep{llama3} as answer models.
For fairness, we keep the answer-generating procedure fixed across all portfolio and baseline methods and use the same prompt formats for all evaluated retrievers. Prompt templates are provided in \cref{app:prompts}.

\subsection{Datasets}
\label{sec:exp:datasets}

We follow the benchmark suite used by Adaptive-RAG \citep{jeong2024adaptive} and evaluate on:
HotpotQA \citep{hotpotqa}, 2WikiMultiHopQA \citep{wikimultihopqa}, TriviaQA \citep{joshi2017triviaqa}, and MusiQue \citep{musique}.
These datasets span factoid and multi-hop questions and include gold supporting documents for retrieval evaluation.
Table~\ref{tab:dataset_stats} summarizes the train/test sizes and corpus statistics.

\begin{table}[t]
\centering
\footnotesize
\setlength{\tabcolsep}{3pt} 
\renewcommand{\arraystretch}{0.9} 
\caption{Dataset statistics. ``Train'' and ``Test'' columns correspond to the number of train and test questions in each dataset. The ``Corpus'' column presents the number of text chunks in the document corpus used for this dataset. ``Avg. Supp.'' is the average number of supporting documents per question.}
\label{tab:dataset_stats}
\resizebox{\columnwidth}{!}{
\begin{tabular}{lcccc}
\toprule
\textbf{Dataset} & \textbf{Train} & \textbf{Test} & \textbf{Corpus} & \textbf{Avg. Supp.} \\
\midrule
HotpotQA & 88{,}066 & 9{,}786 & 5{,}233{,}329 & 2.00 \\
MusiQue          & 19{,}938  & 2{,}417  & 101{,}962     & 2.37 \\
TriviaQA         & 60{,}413  & 8{,}837  & 856{,}435     & 11.86 \\
2WikiMultiHopQA  & 167{,}454 & 12{,}576 & 384{,}857     & 2.42 \\
\bottomrule
\end{tabular}
}
\end{table}

\subsection{Corpus processing and dense indexing}
\label{sec:exp:index}

All retrieval methods operate over a shared chunked corpus, but dense retrieval artifacts are maintained separately for each embedding backbone.
We split each corpus document into overlapping text chunks (512 tokens with 50-token overlap) and treat each chunk as a retrieval unit.
We build one dense index per dataset and backbone.
For MPNet-based retrievers we use the MPNet chunk/query embeddings; for E5-based retrievers we use the corresponding E5 embeddings.
All embeddings are $\ell_2$-normalized, so inner product corresponds to cosine similarity, and each dense index is implemented as a FAISS inner-product index \citep{faiss}.

To make evaluation over large retriever pools tractable, we use cached candidate generation.
For each query and dense backbone we first \emph{prefilter} a fixed candidate set of $M$ text chunks via the matching FAISS index (we use $M=1000$ in all main experiments).
Dense candidate retrievers then re-rank / re-select from this backbone-specific candidate set and output a fixed number of chunks $n$ (we use $n=4$) that are passed to the answer model.
Candidate retrieval outputs are cached per backbone and reused across all retriever configurations; see \cref{app:indexing}.

\subsection{Candidate retriever pools}
\label{sec:exp:pools}

We instantiate the abstract candidate set $\mathcal{R}$ as a heterogeneous union of multiple retriever families and multiple embedding backbones.
Each family/backbone/hyperparameter configuration is treated as a distinct candidate retriever in the portfolio-selection problem.
The portfolio objective and \cref{alg:efficient_greedy} are retriever-agnostic: they only require the training-query score matrix over candidates.

\paragraph{Dense baselines.}
For each dense backbone, we include the standard top-$n$ dense retriever that simply returns the highest-scoring FAISS neighbors with no diversification.
In our implementation this is the DS configuration with hyperparameters $\gamma=0$ and $r=1$, so the dense baselines participate directly in the same candidate pool as the diversified DS variants.

\paragraph{DiscountedSimilarity (DS).}
The DS retriever begins by ranking the backbone-specific candidate chunks according to their similarity scores with the question.
It then iteratively selects $n$ chunks while down-weighting candidates that are too similar to already-selected chunks.
It is parameterized by a discount strength parameter $\gamma$ and a similarity threshold $r$.
For each of MPNet and E5, the DS pool contains $140$ $(\gamma,r)$ settings plus the dense baseline, giving $141$ candidates per backbone embedding.
In \cref{app:ds} we provide a pseudocode implementation of the DS retriever as well as the exact hyperparameter pool and other implementation details.

\paragraph{Vendi retriever.}
Vendi \citep{rezaei2025vendi} selects chunks by trading off query relevance and intra-set diversity using a Vendi-score-based objective, with trade-off parameter $s\in[0,1]$.
We instantiate Vendi separately over MPNet and E5 and sweep $s$ in increments of $0.05$ (21 configurations per backbone).
In \cref{app:vendi} we provide a pseudocode implementation of the Vendi retriever and other implementation details.

\paragraph{Graph-dense retrievers.}
We also include graph-based dense retrieval variants, inspired by \cite{graphrag}.
For each dataset, we build an entity-chunk graph from the chunked corpus, by identifying the entities that exist in a specific document chunk. The indexed graph is simply a bipartite graph where an entity has an edge with a chunk if it appears in that chunk.
At query time, we identify the entities of the query and do a breadth-first search on the indexed graph to gather related document chunks.
The fetched chunks are then reranked with MPNet or E5 dense similarity.
The graph-dense pool sweeps the re-ranking backbone, the maximum number of graph-expansion hops, the maximum entity document frequency, and maximum graph candidates (see \cref{app:pools}), yielding $36$ graph-based candidates.

Overall, all the retriever pools contain $360$ retrievers together: DS (including dense baselines) over MPNet/E5, Vendi over MPNet/E5, and graph-dense variants over MPNet/E5.

\subsection{Portfolio construction}
\label{sec:exp:portfolio}

Portfolios are constructed offline using Algorithm~\ref{alg:efficient_greedy} with the best-of-$k$ objective.
Concretely, we compute a score matrix $s(q,r)$ which represents the performance of all candidate retrievers across a representative sample of training queries.
The score function $s(q,r)$ is instantiated to retrieval Recall@$k$ and is calculated using the supporting documents of each question in the training set.
For the all-pool experiment, we concatenate the score matrices for all family/backbone pools along the retriever dimension and run the greedy objective over the resulting union.
This computation is accelerated by per-backbone caches and batched retriever implementations (for more details see \cref{app:batch}).

Given the score matrix, we greedily select a portfolio $S$ of size $k$ maximizing
$\mathbb{E}_{q}[\max_{r\in S} s(q,r)]$.
Our main portfolio is \emph{union-trained}: we pool training queries across all datasets and select from the full heterogeneous candidate set.
We compare this portfolio against average-best retriever selection, which chooses candidates by their mean training score, and against controls that retrieve more documents with a single retriever.

\subsection{Metrics}
\label{sec:exp:metrics}

We evaluate retrieval quality via \textbf{Support Recall@$k$} and \textbf{F1@$k$} based on ground-truth documents of the datasets.
Downstream QA is assessed using \textbf{Exact Match (EM)} scores alongside generation token counts and wall-clock latency to characterize efficiency-accuracy trade-offs.
We additionally report ``retrieve more documents'' controls to distinguish portfolio complementarity from simply increasing the context budget.
See \cref{app:metrics} for formal definitions, normalization procedures, and efficiency benchmarking details.




\section{Results}
\label{sec:results}

We evaluate retriever portfolios on four QA benchmarks (HotpotQA, 2WikiMultiHopQA, TriviaQA, and MusiQue) using two
answer models (Gemma-3-27B-It and Llama-3.1-70B-Instruct).
Our results address three questions:
(1) do learned portfolios provide better \emph{retrieval coverage} as the portfolio size $k$ grows
(2) do these retrieval gains translate into end-to-end QA improvements beyond ``retrieving more documents,'' and
(3) how do fixed portfolios compare to inference-time adaptation in terms of accuracy--cost trade-offs.

\subsection{Portfolios improve retrieval coverage}
\label{sec:results:retrieval}

\begin{figure*}[t!]
  \centering

  \includegraphics[width=0.98\textwidth]{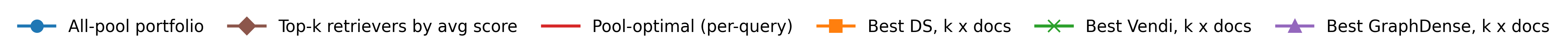}
  \vspace{0.15em}

  \begin{subfigure}[t]{0.48\textwidth}
    \centering
    \includegraphics[width=\linewidth]{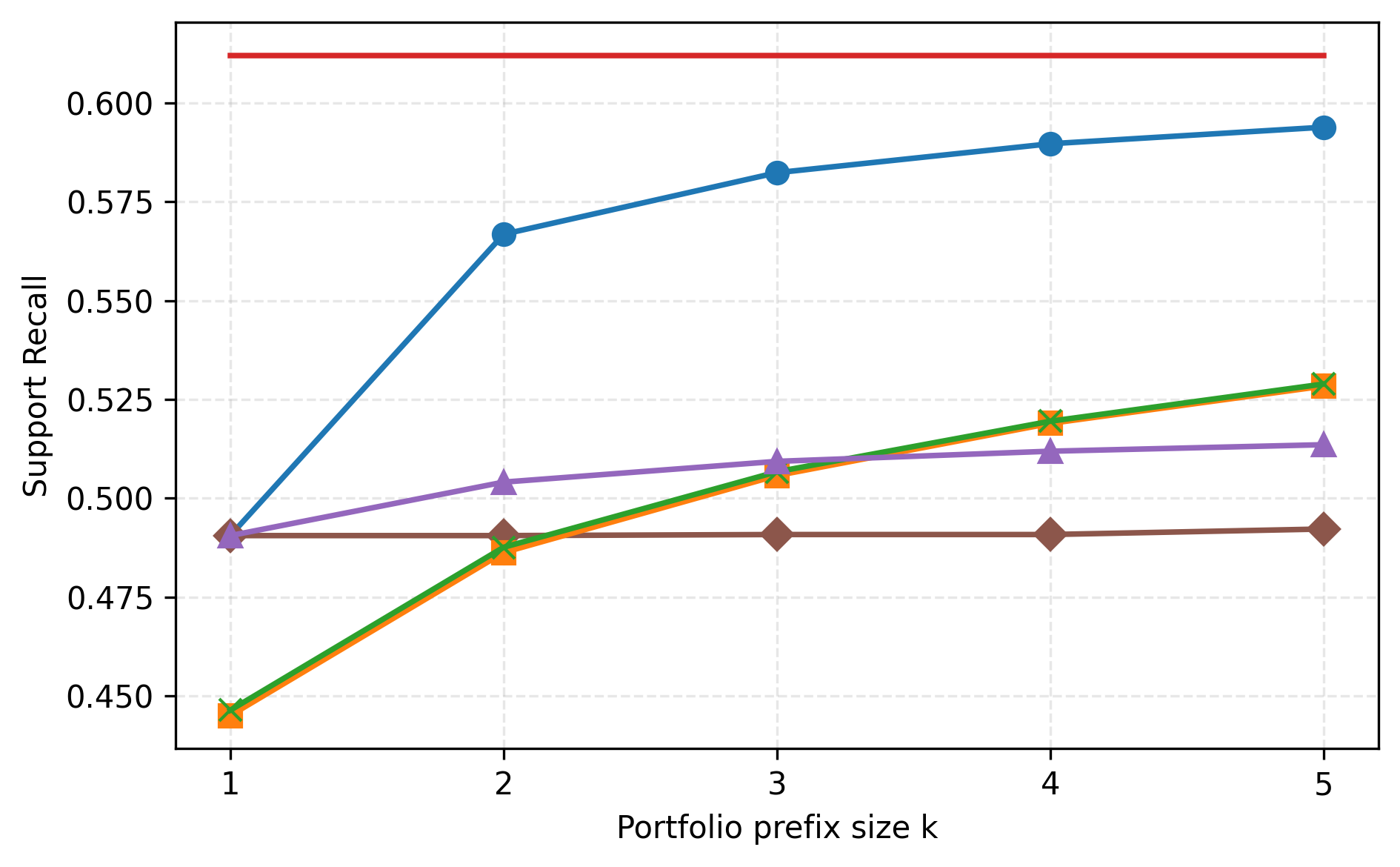}
    \caption{Support Recall}
    \label{fig:all_pool_support_recall}
  \end{subfigure}\hfill
  \begin{subfigure}[t]{0.48\textwidth}
    \centering
    \includegraphics[width=\linewidth]{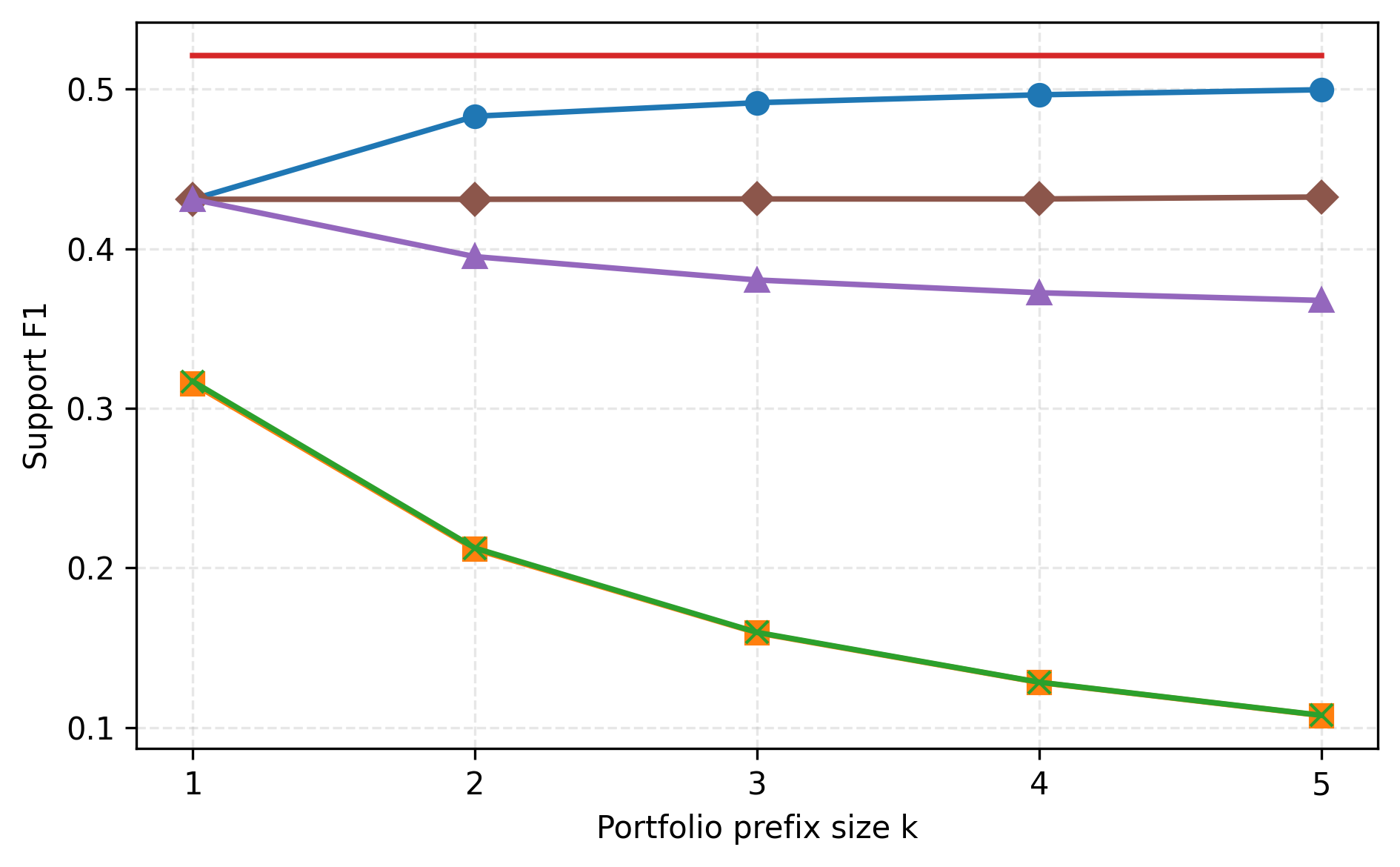}
    \caption{Support F1}
    \label{fig:all_pool_support_f1}
  \end{subfigure}

  \caption{Retrieval coverage of the union-trained all-pool portfolio.
  Curves are averaged over HotpotQA, MusiQue, TriviaQA, and 2WikiMultiHopQA.
  The all-pool curve evaluates greedy portfolio prefixes by best-of-$k$ support score.
  The top-$k$ average baseline selects the $k$ retrievers with highest mean training score.
  The $k\times$ document baselines run one best retriever from a family while increasing its document budget from 4 to $4k$.}
  \label{fig:recalls}
\end{figure*}

We first measure retrieval performance in isolation, independent of the answer model.
Following our formulation, we evaluate a size-$k$ portfolio by its \emph{best-of-$k$} retrieval score:
for each query, we take the maximum support-document score achieved by any member of the portfolio.
This isolates portfolio coverage from the downstream router and answer model, which are evaluated in the next subsection.
Figure~\ref{fig:recalls} reports support recall and support F1 for various portfolio sizes selected from the full heterogeneous candidate pool:
DS and Vendi retrievers with MPNet and E5 backbones, together with GraphDense retrievers, for 360 candidates in total.
The portfolio is trained once on the pooled training queries from all four benchmarks and then evaluated on the corresponding test sets.

\begin{table}[t]
  \caption{Members of the size-5 portfolio used in Figure~\ref{fig:recalls}.
  The portfolio is selected greedily on the union of training queries.}
  \label{tab:all_pool_portfolio}
  \centering
  \scriptsize
  \setlength{\tabcolsep}{2pt}
  \begin{tabular}{@{}rp{0.28\columnwidth}rp{0.38\columnwidth}@{}}
    \toprule
    Rank & Retriever & Avg. Score & Hyperparameters \\
    \midrule
    1 & GraphDense / E5 & 0.6058 & hops=1, df=500 \\
    2 & Vendi / E5 & 0.5421 & $s=0.05$ \\
    3 & GraphDense / E5 & 0.5532 & hops=3, df=100\\
    4 & GraphDense / MPNet & 0.4985 & hops=3, df=100 \\
    5 & Vendi / E5 & 0.5223 & $s=0.20$ \\
    \bottomrule
  \end{tabular}
\end{table}

\paragraph{Portfolio selection is not equivalent to picking the best retrievers on average.}
A natural alternative to portfolio optimization is to perform a grid search over retriever configurations, rank candidates by their average training performance, and keep the top $k$.
Figure~\ref{fig:recalls} evaluates exactly this baseline by taking the best-of-$k$ score over the top-$k$ candidates by average training score.
The baseline is nearly flat: at $k=5$, it reaches only 0.492 support recall and 0.432 support F1, while the learned portfolio reaches 0.594 and 0.500, respectively.
The reason is visible in the selected portfolio.
After the strongest average retriever, the greedy objective adds lower-average but complementary Vendi and GraphDense variants (Table~\ref{tab:all_pool_portfolio}) because they cover queries missed by earlier members.
In contrast, the top-$k$ average list is dominated by closely related GraphDense/E5 configurations, so additional members add little new coverage.

\paragraph{Gains are not explained by retrieving more documents.}
Another possible explanation for the improved recall scores could be that a portfolio simply fetches more context.
The $k\times$ document baselines in Figure~\ref{fig:recalls} test this directly:
they run the single best retriever from each family while increasing the number of returned documents from $4$ to $4\cdot k$.
As we can see, a portfolio with $2$ members significantly outperforms even fetching $20$ documents with the single best retriever.
Inflating the number of fetched documents can increase recall, but does so by admitting many more non-support documents.
Their support F1 therefore drops sharply as $k$ grows (for DS and Vendi, from about 0.32 at $k=1$ to about 0.11 at $k=5$; for GraphDense, from 0.43 to 0.37).
The pool portfolio shows the opposite pattern: support recall rises from 0.491 to 0.594, and support F1 rises from 0.431 to 0.500.
Thus the gains come from selecting retrievers with complementary failure modes, not merely from increasing the number of retrieved passages.

\paragraph{The portfolio captures most of the available support signal.}
The red horizontal line in Figure~\ref{fig:recalls} is the recall of a method that picks the best retriever in the full candidate pool for every query.
It is, of course, not a deployable method, since it assumes knowing which retriever will work for each query, but it indicates the maximum support score available in the pool.
By $k=5$, the learned portfolio closes most of this gap on both metrics, while using only five fixed retriever configurations that can be routed or run under a bounded inference budget.

\subsection{End-to-end QA accuracy with routed portfolios}
\label{sec:results:qa}


\begin{table*}[t!]
\centering
\footnotesize
\setlength{\tabcolsep}{2.8pt}
\renewcommand{\arraystretch}{0.85}
\caption{Answer Exact-Match (EM) across datasets and retrieval strategies for Gemma-3-27B-It and Llama-3.1-70B-Instruct.
Best DS, Vendi, and GraphDense rows select the best single retriever by average performance.
The reported Vendi-only portfolio point, $(k{=}5,\ell{=}2)$, and all-pool operating points, $(k{=}4,\ell{=}2)$ and $(k{=}4,\ell{=}3)$, are selected from their routed sweeps by average EM.
Bold denotes the best value in each dataset/model column.
}
\label{tab:em}
\resizebox{\textwidth}{!}{%
\begin{tabular}{lcccccccc}
\toprule
  & \multicolumn{4}{c}{\textbf{Model: Gemma-3-27B-It}} & \multicolumn{4}{c}{\textbf{Model: Llama-3.1-70B-Instruct}} \\
\cmidrule(lr){2-5} \cmidrule(lr){6-9}
\textbf{Method}
  & \multicolumn{1}{c}{\textbf{HotpotQA}}
  & \multicolumn{1}{c}{\textbf{MusiQue}}
  & \multicolumn{1}{c}{\textbf{TriviaQA}}
  & \multicolumn{1}{c}{\textbf{\shortstack{2Wiki-\\MultiHopQA}}}
  & \multicolumn{1}{c}{\textbf{HotpotQA}}
  & \multicolumn{1}{c}{\textbf{MusiQue}}
  & \multicolumn{1}{c}{\textbf{TriviaQA}}
  & \multicolumn{1}{c}{\textbf{\shortstack{2Wiki-\\MultiHopQA}}} \\
\midrule
No retrieval                      & 0.3256 & 0.0608 & 0.6530 & 0.2264 & 0.3483 & 0.0588 & 0.6947 & 0.1920 \\
Nearest-neighbor retrieval (MPNet) & 0.3945 & 0.1291 & 0.6592 & 0.2408 & 0.4763 & 0.1394 & 0.7022 & 0.2921 \\
\addlinespace[0.4em]
Best DS Retriever                 & 0.5128 & 0.1394 & 0.6867 & 0.3544 & 0.4345 & 0.1092 & 0.6318 & 0.2443 \\
Best Vendi Retriever              & 0.5111 & 0.1427 & 0.6868 & 0.3562 & 0.4331 & 0.1121 & 0.6249 & 0.2445 \\
Best GraphDense Retriever         & 0.4628 & 0.1134 & 0.6704 & 0.3407 & 0.3978 & 0.0873 & 0.5961 & 0.2387 \\
\addlinespace[0.4em]
Vendi-RAG Adaptive ($T$=10)\  (MPNet)        & 0.2813 & 0.1253 & 0.6720 & 0.2590 & 0.4860 & 0.1990 & 0.7030 & 0.2890 \\
Vendi-RAG Adaptive ($T$=20) \ (MPNet)       & 0.2853 & 0.1307 & 0.6727 & 0.2560 & 0.4830 & 0.2060 & 0.7040 & 0.2900 \\
Vendi-Portfolio ($k=5$, $\ell=2$) (MPNet) & 0.3927 & 0.1262 & 0.6720 & 0.2341 & 0.4825 & 0.1427 & 0.6956 & 0.3080 \\
\addlinespace[0.4em]
All-pool Portfolio ($k=4$, $\ell=2$) & 0.5523 & 0.1725 & \textbf{0.6937} & 0.4052 & \textbf{0.5897} & 0.1820 & \textbf{0.7073} & 0.4139 \\
All-pool Portfolio ($k=4$, $\ell=3$) & \textbf{0.5579} & \textbf{0.1949} & 0.6898 & \textbf{0.4142} & 0.5827 & \textbf{0.2089} & 0.6859 & \textbf{0.4187} \\
\bottomrule
\end{tabular}
}
\end{table*}

We now evaluate end-to-end Exact Match (EM), where a retrieved context is fed to the answer LLM.
Table~\ref{tab:em} summarizes EM for no retrieval, nearest-neighbor retrieval, the best single retriever from each family, Vendi-RAG inference-time adaptation, a Vendi-only routed portfolio, and routed all-pool portfolios.
The single-family rows choose the best configuration by average performance.
For the routed portfolio rows, we report average-best operating points from their sweeps: Vendi-only $(k{=}5,\ell{=}2)$, and all-pool $(k{=}4,\ell{=}2)$ and $(k{=}4,\ell{=}3)$.

\paragraph{All-pool portfolios dominate the fixed-budget comparison.}
The strongest entry in every dataset/model column comes from a full-pool portfolio.
For Gemma-3-27B-It, the all-pool $(k{=}4,\ell{=}3)$ configuration has the best average EM, $0.464$, while $(k{=}4,\ell{=}2)$ is close at $0.456$.
Both exceed the best single-family retrievers selected by average performance (DS: $0.423$, Vendi: $0.424$, GraphDense: $0.397$) and nearest-neighbor retrieval ($0.356$).
For Llama-3.1-70B-Instruct, the same pattern holds: all-pool $(k{=}4,\ell{=}3)$ reaches average EM $0.474$, with $(k{=}4,\ell{=}2)$ essentially tied at $0.473$, above nearest-neighbor retrieval ($0.403$) and Vendi-RAG with $T=20$ ($0.421$).

\paragraph{Heterogeneity matters.}
The family-best DS, Vendi, and GraphDense baselines are strong single-retriever controls: each selects the best configuration within a retriever family by average EM. However, their performance varies substantially across datasets and answer models, so no single retriever family provides the best choice for all queries. The all-pool portfolio is designed to exploit this variation by selecting from the union of dense, diversity-based, and graph-based retrievers and routing each query to a small subset of portfolio members. Its consistent average gains over the family-best rows are therefore consistent with cross-family complementarity rather than merely better tuning within one family. This mirrors the retrieval-only result in Figure~\ref{fig:recalls}: average-best retriever selection identifies strong individual configurations, while the portfolio objective favors candidates that cover different query subpopulations.

\paragraph{Routing budget is a cost-accuracy knob.}
The two reported all-pool rows differ only in the number of routed portfolio members executed per query.
Increasing from $\ell=2$ to $\ell=3$ gives the best average EM for both answer models, especially on MusiQue and 2WikiMultiHopQA, while $\ell=2$ is competitive and wins some columns such as TriviaQA.
In deployment, this exposes a simple trade-off: $(k{=}4,\ell{=}3)$ is the highest-accuracy operating point in the table, while $(k{=}4,\ell{=}2)$ preserves most of the gain with fewer answer generations and judge comparisons.

\vspace{-0.2em}
\subsection{Efficiency-accuracy trade-offs}
\label{sec:results:efficiency}

\begin{figure*}[t!]
  \centering

  \includegraphics[width=0.8\textwidth]{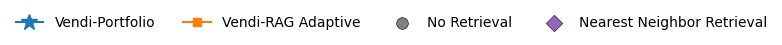}
  \vspace{0.15em}

  \begin{subfigure}[t]{0.42\textwidth}
    \centering
    \includegraphics[width=\linewidth]{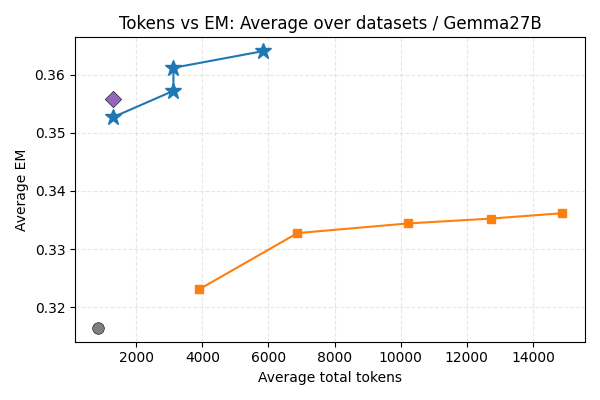}
    \caption{Tokens vs.\ EM (Gemma27B)}
    \label{fig:tokens_em_gemma}
  \end{subfigure}\hfill
  \begin{subfigure}[t]{0.42\textwidth}
    \centering
    \includegraphics[width=\linewidth]{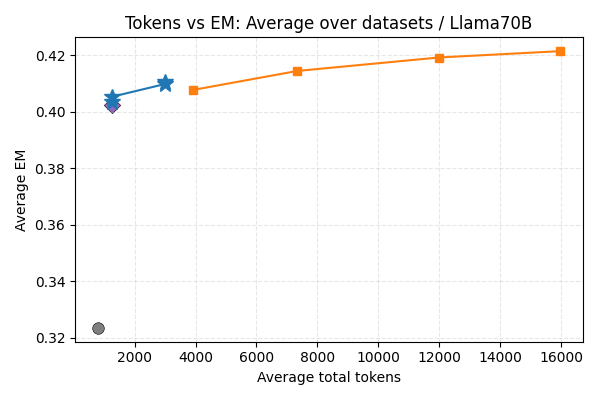}
    \caption{Tokens vs.\ EM (Llama70B)}
    \label{fig:tokens_em_llama}
  \end{subfigure}

  \vspace{0.6em}

  \begin{subfigure}[t]{0.42\textwidth}
    \centering
    \includegraphics[width=\linewidth]{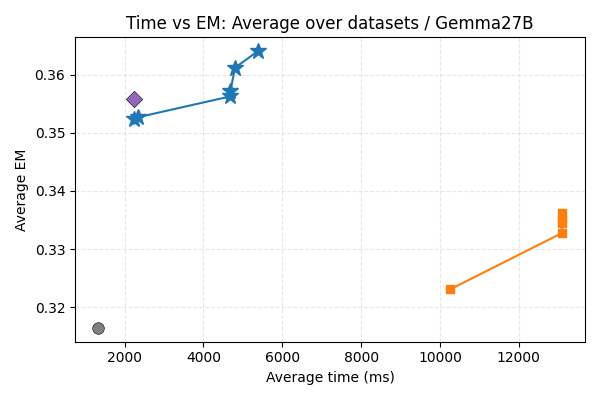}
    \caption{Time vs.\ EM (Gemma27B)}
    \label{fig:time_em_gemma}
  \end{subfigure}\hfill
  \begin{subfigure}[t]{0.42\textwidth}
    \centering
    \includegraphics[width=\linewidth]{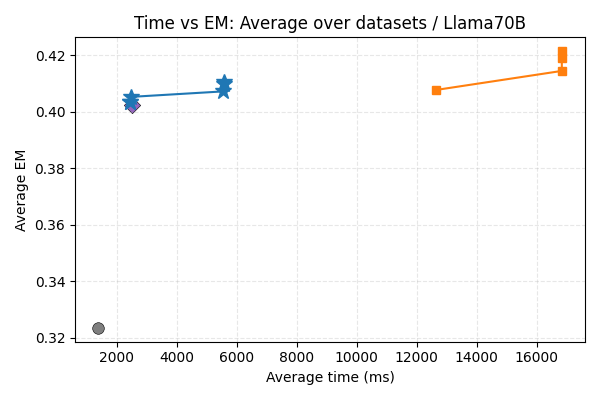}
    \caption{Time vs.\ EM (Llama70B)}
    \label{fig:time_em_llama}
  \end{subfigure}

  \caption{Efficiency--accuracy trade-offs. Row 1: total generated tokens vs.\ EM. Row 2: wall-clock time vs.\ EM.
  Columns correspond to Gemma-3-27B-It and Llama-3.1-70B-Instruct.}
  \label{fig:efficiency_grid_shared_legend}
\end{figure*}

Accuracy alone is insufficient for deployable RAG systems; adaptation must also be efficient.
For this comparison, we intentionally do \emph{not} use the heterogeneous all-pool portfolios from Figure~\ref{fig:recalls}.
Instead, to make the comparison with Vendi-RAG as direct as possible, we isolate a single embedding backbone (MPNet) and a single retriever family (Vendi).
The portfolio pool contains only Vendi retrievers that differ in the diversity parameter $s$.
This is exactly the one-dimensional hyperparameter that Vendi-RAG attempts to tune adaptively at inference time.
Thus Figure~\ref{fig:efficiency_grid_shared_legend} should be read as an apples-to-apples comparison between two ways of using the same Vendi search space: selecting a small fixed portfolio of $s$ values offline versus searching over $s$ sequentially at inference time.
It is not intended as a full comparison between Vendi-RAG and our largest heterogeneous portfolios.

\paragraph{Offline Vendi portfolios are more efficient in the controlled comparison.}
Under this restricted MPNet/Vendi setting, fixed Vendi portfolios occupy a better region of the cost-accuracy space:
they achieve comparable EM to Vendi-RAG in the low-to-moderate cost regime while using substantially fewer tokens
and lower wall-clock time (Figure~\ref{fig:efficiency_grid_shared_legend}).
On Gemma, the portfolio is strictly preferable in this comparison, achieving higher EM at lower cost.
On Llama, Vendi-RAG can reach slightly higher EM, but only after paying a large increase in tokens and latency.

\paragraph{Why portfolios are cheaper.}
The search space is the same in this comparison; the difference is when and how it is searched.
Vendi-RAG performs \emph{sequential} per-query search over $s$: it iterates retrieval, generation and judging, and
\emph{each step depends on the previous step's judged output}.
In contrast, portfolio learning amortizes the search offline by choosing a small set of complementary $s$ values.
At inference time, routing over the fixed Vendi portfolio requires only a bounded number of retrievals and LLM calls,
and these calls can be parallelized across the selected portfolio members.
This yields more predictable serving costs and enables practical deployment under latency/token constraints.

\paragraph{Summary.}
Taken together, Figure~\ref{fig:recalls} shows that portfolio learning yields complementary retrievers that better cover heterogeneous
queries, Table~\ref{tab:em} shows that these gains translate into improved end-to-end QA (especially for Gemma), and Figure~\ref{fig:efficiency_grid_shared_legend} shows
that, in the controlled Vendi-only setting, offline portfolios provide a substantially better cost--accuracy trade-off than inference-time tuning over the same $s$ grid.

\section*{Impact Statement}

This work presents a principled framework for query-adaptive information retrieval, advancing the field of Retrieval-Augmented Generation (RAG). No specific malicious use of this technology is expected beyond the general risks associated with the deployment of large-scale generative models.

\section*{Acknowledgements}
Miltiadis Stouras and Ola Svensson are supported by the Swiss State Secretariat for Education, Research and Innovation (SERI) under contract number MB22.00054.


\bibliography{bibliography}
\bibliographystyle{icml2026}

\newpage
\appendix
\onecolumn

\section{Router Ablations}
\label{app:router-ablations}

We ablate the learned router on the all-pool portfolio by comparing its ranked
predictions against three references: a random portfolio member, the first
portfolio member, and the oracle maximum over the portfolio. The first member is
the size-one portfolio selected by the greedy objective, and therefore
corresponds to the strongest single-retriever operating point within this
portfolio prefix. The recall curve is shared across answer models; the EM curves
use the corresponding answer model.

\begin{figure}[H]
  \centering

  \includegraphics[width=0.95\textwidth]{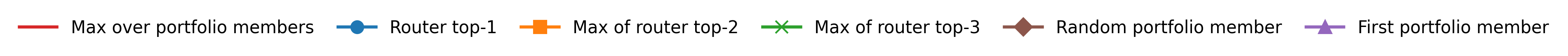}
  \vspace{0.25em}

  \begin{subfigure}[t]{0.55\textwidth}
    \centering
    \includegraphics[width=\linewidth]{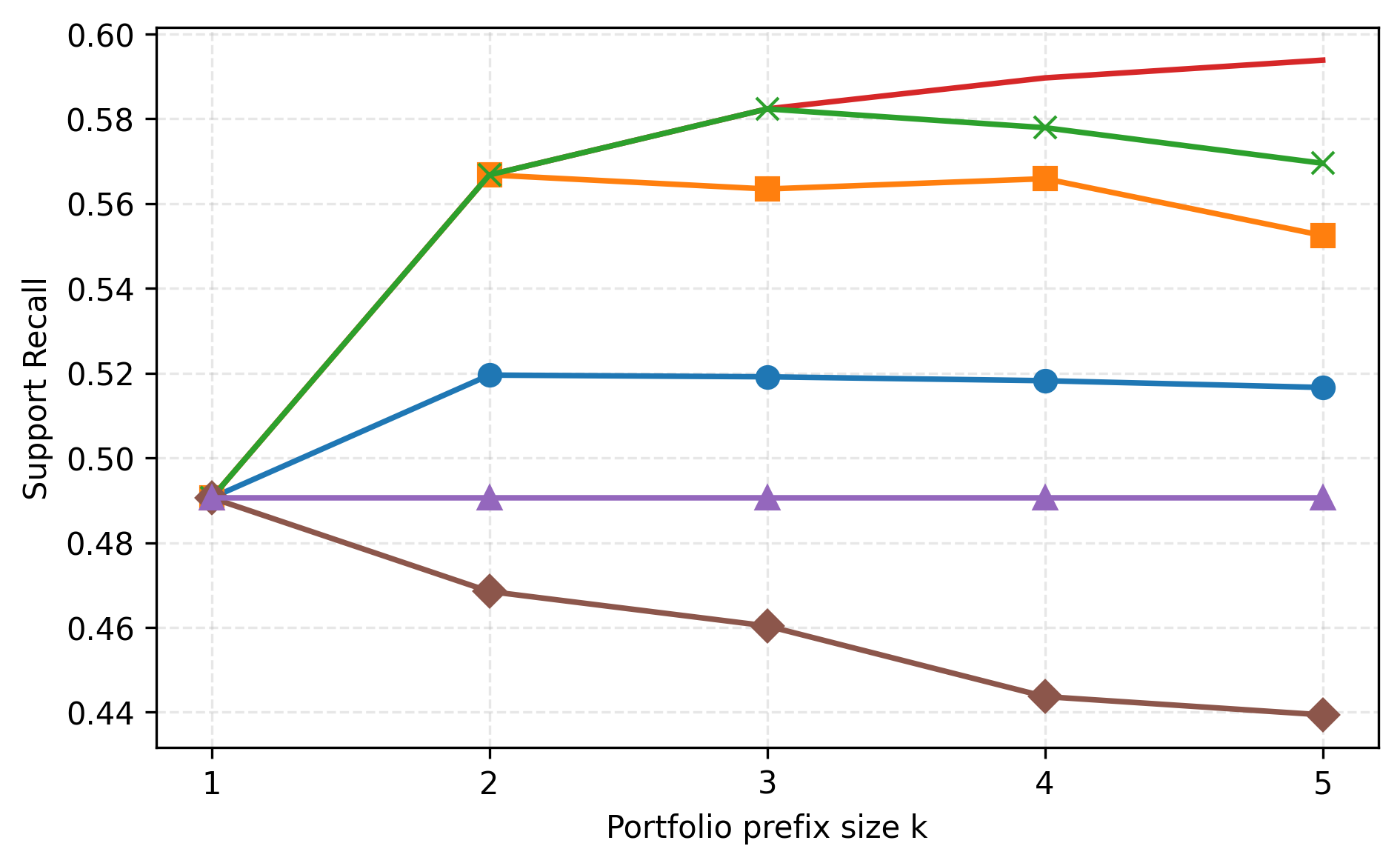}
    \caption{Support recall, common to both answer models.}
    \label{fig:router_ablation_recall}
  \end{subfigure}

  \vspace{0.5em}

  \begin{subfigure}[t]{0.48\textwidth}
    \centering
    \includegraphics[width=\linewidth]{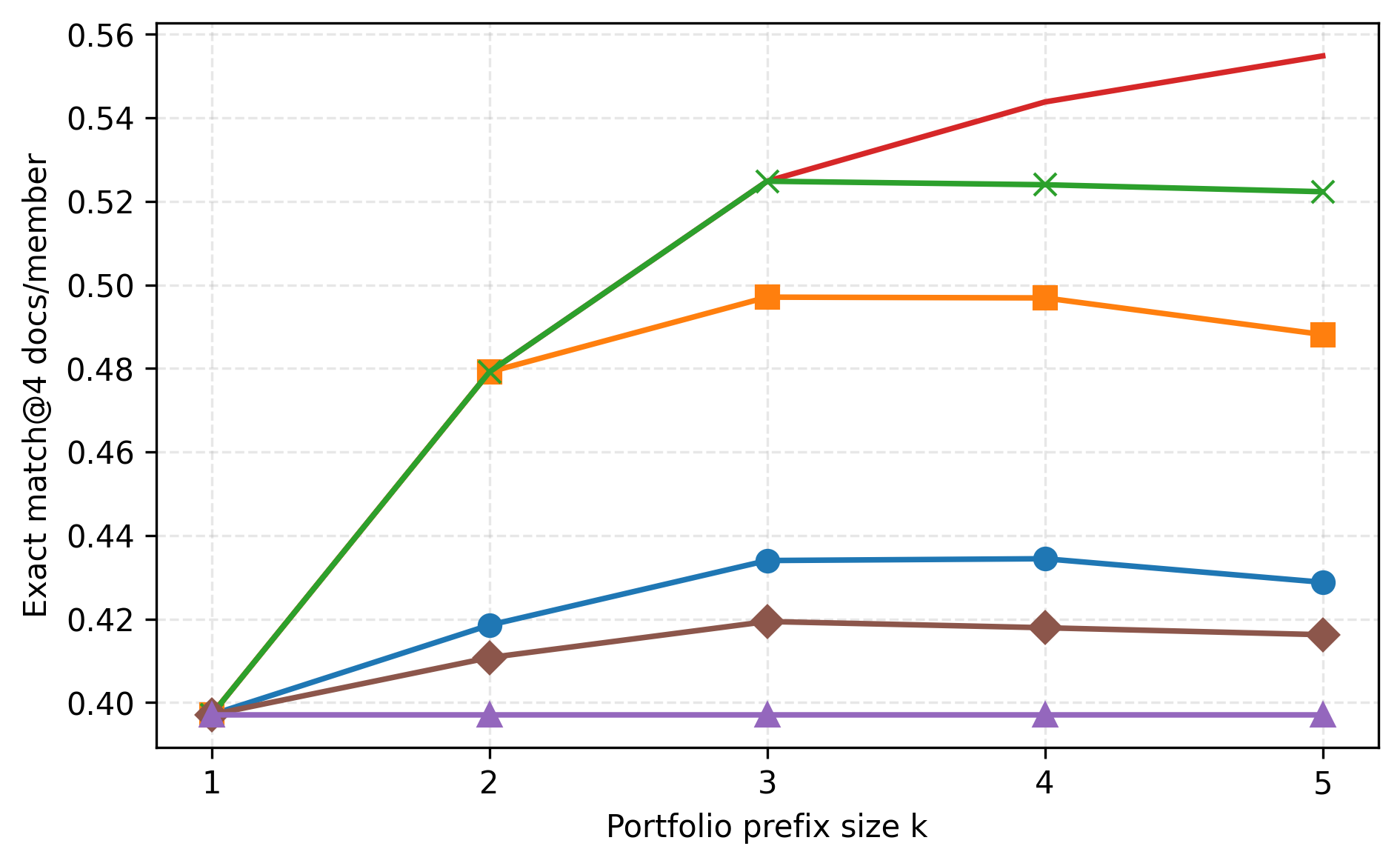}
    \caption{Exact match with Gemma-3-27B-It.}
    \label{fig:router_ablation_gemma_em}
  \end{subfigure}\hfill
  \begin{subfigure}[t]{0.48\textwidth}
    \centering
    \includegraphics[width=\linewidth]{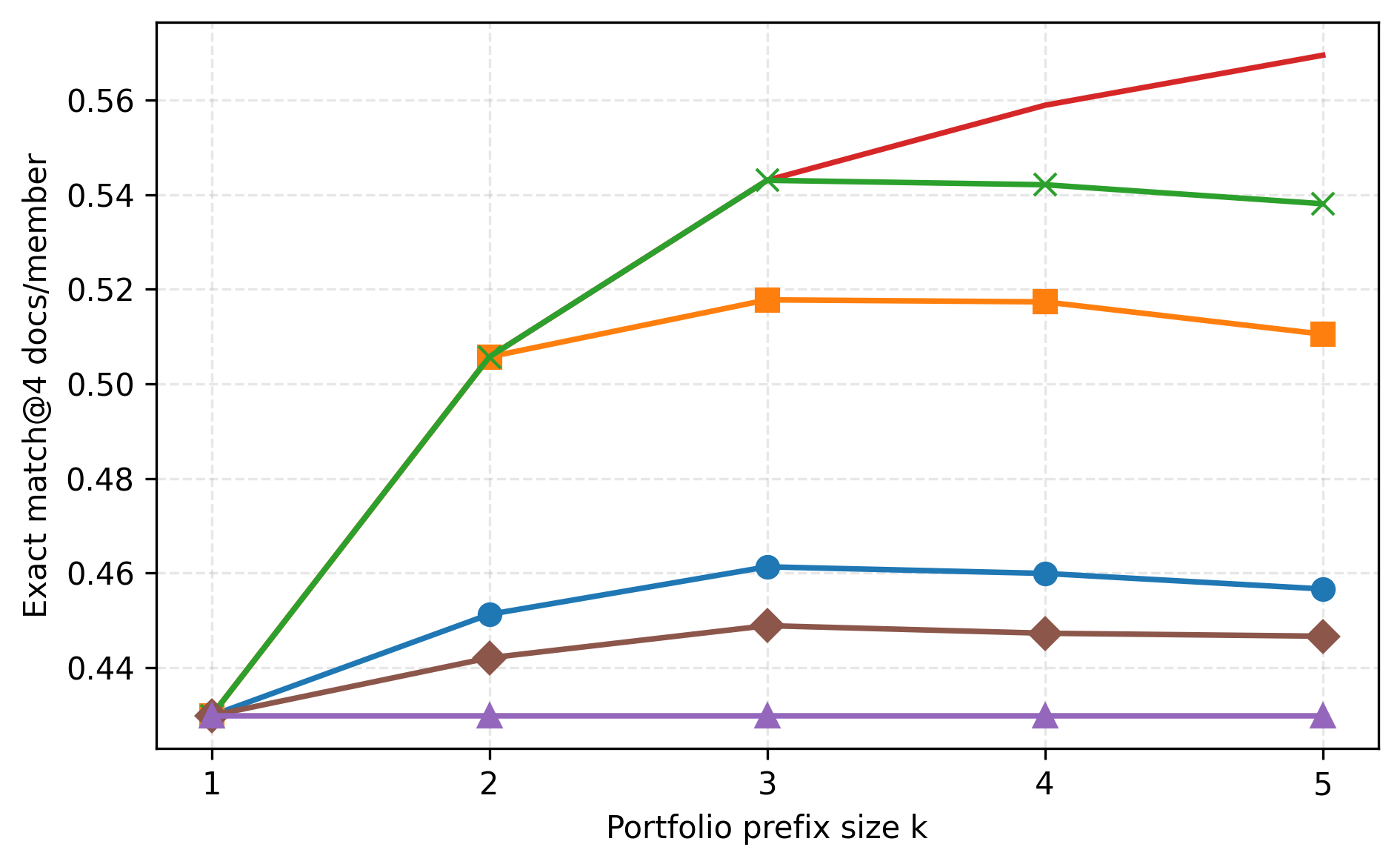}
    \caption{Exact match with Llama-3.1-70B-Instruct.}
    \label{fig:router_ablation_llama_em}
  \end{subfigure}

  \caption{Router ablations for the all-pool portfolio with $4$ documents per
  portfolio member. Curves are averaged over HotpotQA, MusiQue, TriviaQA, and
  2WikiMultiHopQA. The shared legend applies to all panels.}
  \label{fig:router_ablations}
\end{figure}

The router clearly improves over random selection from the portfolio. In the
recall panel, random selection becomes worse as the portfolio grows because it
often chooses a member that is not suitable for the current query. The learned
router avoids this failure mode: even its top-1 prediction stays well above the
random-member baseline for all nontrivial portfolio sizes, and the same pattern
appears in both EM panels.

The top-1 router prediction also consistently improves over the single best
retriever once $k>1$. This shows that the router is not merely defaulting to the
best average member. It uses the query to select different retrievers for
different instances, yielding higher support recall and higher downstream EM
than always using the first portfolio member.

Finally, the top-2 and top-3 curves are close to the oracle maximum over the
portfolio. This indicates that the router often assigns high mass to the
retrievers that are correct for the query, even when its first prediction is not
the best one. This motivates the final selection layer: by executing the top
$\ell$ routed members and using the selector LLM to choose among their answers,
the pipeline can exploit the router's calibrated short list rather than relying
only on a single argmax decision.

\section{Implementation Details}
\label{sec:app-implementation}

This appendix provides additional implementation details needed to reproduce our
retriever-portfolio experiments (Section~\ref{sec:experimental-setup}).
Unless stated otherwise, all experiments use dot-product similarity on
$\ell_2$-normalized embeddings, so dot product coincides with cosine similarity.

\subsection{Corpus preprocessing and FAISS indexing}
\label{app:indexing}

\paragraph{Chunking.}
For each dataset, we build a retrieval corpus by collecting all documents
available in the dataset release (typically the union of train and test
contexts). We split each document into overlapping chunks (``text units'') using
a token-based sliding window with chunk size $512$ and overlap $50$ tokens. Each
text unit stores its \texttt{doc\_id}, chunk id, and text; dense embeddings are
attached when building the corresponding backbone-specific index.

\paragraph{Embeddings.}
We use two dense embedding backbones: MPNet
(\texttt{sentence-transformers/all-mpnet-base-v2}) and E5
(\texttt{intfloat/e5-large-v2}).
For each backbone, we embed both text units and queries with the corresponding
SentenceTransformers encoder and $\ell_2$-normalize their embeddings.
For E5, we use the standard role prefixes: \texttt{query:} for questions and
\texttt{passage:} for chunks.

\paragraph{FAISS indices.}
We build one FAISS IndexFlatIP \citep{faiss} (exact maximum inner-product
search) per dataset and dense backbone.
For MPNet-based retrievers we search the MPNet index with MPNet query
embeddings; for E5-based retrievers we search the E5 index with E5 query
embeddings.
Given a query/backbone pair, we first retrieve a prefiltered candidate set of
size $M$ (we use $M=1000$ in all main experiments), then apply the
retriever-specific selection rule over these candidates to return the final top
$K$ chunks (we use $K=4$ throughout).

\paragraph{Graph sidecar.}
Graph-dense retrievers are a simple graph-retriever variant inspired by
\cite{graphrag}. They use the same chunk identifiers as the dense indices, but
candidate generation is driven by a dataset-level entity--chunk graph rather
than by FAISS prefiltering.
The graph sidecar stores entity-to-chunk and chunk-to-entity adjacency lists and
can attach chunk embeddings from either dense backbone for final dense reranking.

\paragraph{Caching.}
To amortize evaluation over large retriever pools, we cache:
(i) query embeddings separately for each split and backbone, (ii) for each
query/backbone pair, its top-$M$ FAISS candidates (including the candidate
embeddings), and (iii) graph query-entity lists for graph-dense retrieval.
Candidate retrieval outputs and score matrices are cached per family/backbone
pool and reused across all retriever configurations.

\subsection{DiscountedSimilarity retriever}
\label{app:ds}

DiscountedSimilarity (DS) is a greedy diversification heuristic that starts
from standard query-chunk similarity, but down-weights candidates that are too
similar to already-selected chunks.
We instantiate DS separately for each dense backbone; the query embedding,
candidate embeddings, and prefiltered candidate set always come from the same
backbone-specific cache.

Let $\mathcal{C}=\{x_1,\dots,x_M\}$ be the prefiltered candidate embeddings for a
query $q$, and let initial scores be $s_i \leftarrow q^\top x_i$. DS iteratively
constructs a set $S$ of size $k$. After selecting a chunk $x_{i^\star}$, it
computes similarities $\sigma_j \leftarrow x_j^\top x_{i^\star}$ for remaining
candidates and applies a multiplicative discount to those above a threshold
$r$:
\[
s_j \;\leftarrow\; s_j \cdot \exp(-\gamma \, \sigma_j)
\quad\text{for all } j \text{ with } \sigma_j \ge r,
\]
where $\gamma \ge 0$ controls discount strength and $r \in [0,1]$ controls which
candidates are considered ``too similar''.

\begin{algorithm}[t!]
\caption{DiscountedSimilarity retrieval (DS)}
\label{alg:ds}
\begin{algorithmic}[1]
\STATE \textbf{Input:} query embedding $q$, candidate embeddings $\{x_i\}_{i=1}^M$, parameters $(\gamma,r)$, output size $K$
\STATE Initialize scores $s_i \leftarrow q^\top x_i$ for all $i$; $S \leftarrow [\,]$
\FOR{$t=1$ to $K$}
  \STATE $i^\star \leftarrow \arg\max_i s_i$
  \STATE Append $i^\star$ to $S$
  \STATE For each remaining $j$, compute $\sigma_j \leftarrow x_j^\top x_{i^\star}$
  \STATE For each remaining $j$ with $\sigma_j \ge r$, update $s_j \leftarrow s_j \cdot \exp(-\gamma \sigma_j)$
  \STATE Set $s_{i^\star} \leftarrow -\infty$ (remove selected item)
\ENDFOR
\STATE \textbf{Return:} indices $S$
\end{algorithmic}
\end{algorithm}

\paragraph{Naive baseline.}
For each backbone, the DS pool includes a ``naive'' configuration
$(\gamma=0, r=1)$, which reduces to standard top-$k$ dense retrieval by
similarity (no discounting in practice).

\subsection{Vendi retriever}
\label{app:vendi}

The Vendi retriever trades off query relevance with intra-set diversity via a
Vendi-score objective \citep{rezaei2025vendi}. Given a set of $n$ selected
normalized embeddings $X_S \in \mathbb{R}^{n \times d}$, define the kernel
$K_S = X_S X_S^\top$ and let $\lambda_1,\dots,\lambda_n$ be eigenvalues of
$K_S/n$. Define normalized weights $w_i = \lambda_i / \sum_j \lambda_j$ and
their entropy $H(w) = -\sum_i w_i \log w_i$. The Vendi score is
\[
\Vendi(S) \;=\; \textnormal{exp}(H(w)).
\]
Intuitively, $\Vendi(S)$ behaves like an ``effective number of distinct items'',
ranging from $1$ (highly redundant) to $|S|$ (maximally diverse).

\paragraph{Greedy selection with trade-off $s$.}
For a query $q$ and trade-off parameter $s\in[0,1]$, we greedily construct
$S$ by selecting, at each step, the candidate that maximizes:
\[
\text{obj}(x) \;=\; s \cdot \Vendi(S \cup \{x\}) \;+\; (1-s)\cdot \sum_{y\in S\cup\{x\}} q^\top y.
\]
Thus larger $s$ emphasizes diversity, while smaller $s$ emphasizes dense
relevance.
In our batched implementation (Section~\ref{app:batch}), the first choice is the
most relevant item because $\Vendi(\{x\})$ is constant over singleton sets and
the prefiltered candidate order is relevance-sorted.
Subsequent steps follow the objective above.

\begin{algorithm}[t!]
\caption{Vendi retrieval (greedy)}
\label{alg:vendi}
\begin{algorithmic}[1]
\STATE \textbf{Input:} query embedding $q$, candidate embeddings $\{x_i\}_{i=1}^M$ (normalized), trade-off $s\in[0,1]$, output size $K$
\STATE Compute relevance scores $a_i \leftarrow q^\top x_i$
\STATE $i_1 \leftarrow \arg\max_i a_i$; $S \leftarrow [i_1]$; $A \leftarrow a_{i_1}$
\FOR{$t=2$ to $K$}
  \STATE For each remaining candidate $i$, compute $\text{obj}(i)= s\cdot \Vendi(S\cup\{i\}) + (1-s)\cdot (A+a_i)$
  \STATE $i^\star \leftarrow \arg\max_i \text{obj}(i)$; append $i^\star$ to $S$
  \STATE $A \leftarrow A + a_{i^\star}$
\ENDFOR
\STATE \textbf{Return:} indices $S$
\end{algorithmic}
\end{algorithm}

\subsection{Graph-dense retriever}
\label{app:graphdense}

The graph-dense retriever is a simple graph-retriever variant inspired by
\cite{graphrag}. It uses a dataset-level entity--chunk graph for candidate
generation and a dense encoder only for the final reranking step.

\paragraph{Graph construction.}
For each dataset, we run entity extraction on document titles and corpus chunks.
Entities are normalized by lowercasing, collapsing whitespace, and stripping
surrounding punctuation. The graph sidecar stores two adjacency maps:
(i) \texttt{entity\_to\_chunk\_keys}, which maps each normalized entity to the
ordered list of chunks that mention it, and (ii) \texttt{chunk\_to\_entities},
which maps each chunk key $(\texttt{doc\_id},\texttt{chunk\_id})$ to its
normalized entities. Title entities are merged into every chunk from the same
document. The sidecar also stores MPNet and E5 chunk-embedding matrices aligned
to the same chunk-key order.

\paragraph{Query-time retrieval.}
For each train/test split, query entities are extracted once and cached in
question order. At retrieval time, the graph-dense retriever starts from the
cached query entities, alternates entity and chunk frontiers, filters overly
frequent entities, caps the number of graph candidates, and reranks the
resulting chunks by dense similarity under the selected backbone.

\begin{algorithm}[t!]
\caption{Graph-dense retrieval}
\label{alg:graphdense}
\begin{algorithmic}[1]
\STATE \textbf{Input:} $q$, $E_0$, graph $G$, backbone $b$, limits $(H,D,C)$, output size $K$
\STATE $E \leftarrow$ normalized unique entities from $E_0$; $A\leftarrow[\,]$
\STATE $V_E,V_C\leftarrow\emptyset$ \hfill \textit{visited entities and chunks}
\FOR{odd hop $h\leq H$}
  \STATE $B\leftarrow[\,]$ \hfill \textit{new chunks reached in this hop}
  \FOR{entity $e$ in $E$ with $e\notin V_E$ and $1\leq\mathrm{df}(e)\leq D$}
    \STATE Add $e$ to $V_E$
    \FOR{chunk $c$ mentioning $e$, in graph order}
      \IF{$c\notin V_C$}
        \STATE Append $c$ to $B$ and $A$; add $c$ to $V_C$
      \ENDIF
      \IF{$|A|=C$}
        \STATE break
      \ENDIF
    \ENDFOR
    \IF{$|A|=C$}
      \STATE break
    \ENDIF
  \ENDFOR
  \IF{$|A|=C$ or $h=H$}
    \STATE break
  \ENDIF
  \STATE Set $E$ to the normalized, unique, unvisited entities in $B$ with $1\leq\mathrm{df}(e)\leq D$
  \IF{$E$ is empty}
    \STATE break
  \ENDIF
\ENDFOR
\IF{$A$ is empty}
  \STATE \textbf{return} $[\,]$
\ENDIF
\STATE Embed $q$ with backbone $b$ and load backbone-$b$ embeddings for chunks in $A$
\STATE Rank $c\in A$ by inner product $x_c^{(b)\top} q^{(b)}$
\STATE \textbf{Return:} the top-$K$ ranked chunks
\end{algorithmic}
\end{algorithm}

\paragraph{Hyperparameters.}
The graph-dense family has four hyperparameters:
\texttt{embedder} selects the dense reranking backbone, either MPNet or E5;
\texttt{max\_hops} ($H$) is the graph-expansion depth, where odd hops add
chunks from the current entity frontier and even hops discover new entities
from the current chunk frontier; \texttt{max\_entity\_df} ($D$) is a
document-frequency filter that skips entities appearing in more than $D$ chunks; and
\texttt{max\_candidates} ($C$) is the maximum number of graph candidate chunks
kept before dense reranking. If no query entities or no graph candidates are
available, the retriever returns no chunks in our graph-dense pool evaluation.

\subsection{Candidate retriever pools}
\label{app:pools}

We form structured candidate pools by sweeping hyperparameters and embedding
backbones. The main all-pool portfolio concatenates the catalogued pools
\texttt{ds@mpnet}, \texttt{ds@e5}, \texttt{vendi@mpnet},
\texttt{vendi@e5}, and \texttt{graph\_dense@mixed}.

\paragraph{DS pools.}
For each of MPNet and E5, we sweep
$\gamma \in \{0.2,0.4,0.6,0.8,1.0,1.2,1.4,1.6,1.8,2.0,4,6,8,10\}$ and
$r \in \{0.0,0.1,\dots,0.9\}$, yielding $14\times 10=140$ diversified DS
configurations. We add the naive setting $(0,1)$ for each backbone, for a total
of $141$ DS retrievers per backbone.

\paragraph{Vendi pools.}
For each of MPNet and E5, we sweep
$s \in \{0.0, 0.05, 0.10, \dots, 0.95, 1.0\}$, giving $21$ Vendi retrievers per
backbone.

\paragraph{Graph-dense pool.}
Graph-dense candidates use the retriever in \cref{app:graphdense}. We sweep
the reranking backbone in $\{\text{MPNet},\text{E5}\}$,
\[
\texttt{max\_hops}\in\{1,3,5\},\qquad
\texttt{max\_entity\_df}\in\{100,300,500\},\qquad
\texttt{max\_candidates}\in\{1000,2000\}.
\]
This gives $2\times 3\times 3\times 2=36$ graph-dense variants. The pool is
stored under a mixed graph-dense artifact key because one graph family contains
both MPNet-reranked and E5-reranked variants. Individual variants are named
\texttt{\{embedder\}\_h\{max\_hops\}\_df\{max\_entity\_df\}\_c\{max\_candidates\}}.

Altogether, the all-pool candidate set contains $360$ retrievers.

\subsection{Efficient batched evaluation over retriever pools}
\label{app:batch}

Evaluating all retriever configurations for all queries can be dominated by
repeated reranking and greedy selection. We therefore batch computation within
each family/backbone pool and reuse cached artifacts whenever possible.

\paragraph{Backbone-specific prefiltering.}
For each query and dense backbone, we run FAISS once to obtain top-$M$
candidates and reuse this fixed set across all DS and Vendi configurations that
use that backbone. This eliminates repeated nearest-neighbor search when
sweeping hyperparameters while keeping MPNet and E5 artifacts separate.

\paragraph{BatchDiscountedSimilarity.}
Given a fixed candidate set, we precompute the candidate embedding matrix
$X\in\mathbb{R}^{M\times d}$ and the base similarity vector $a=Xq$ once. For a batch
of $B$ DS configurations $\{(\gamma_b,r_b)\}_{b=1}^B$, we maintain a score matrix
$S\in\mathbb{R}^{B\times M}$ initialized as $S_{b,i}=a_i$, and update all $B$
rows in parallel when discounting around each selected item.
These matrix operations are highly parallelizable and can be significantly accelerated via GPU kernels, 
providing a large speedup when $B$ is large (e.g., 141 DS configurations per
backbone).

\paragraph{BatchVendiRetriever.}
For a batch of $B$ Vendi configurations (different $s$ values), we share:
(i) relevance scores $a_i = q^\top x_i$, and (ii) the candidate-candidate similarity kernel
$K = XX^\top$. Since $k$ is small ($M=1000$) and we only ever compute the Vendi
score on subsets of size up to $k=4$, the incremental cost is dominated by
eigendecompositions of $3\times 3$ and $4\times 4$ matrices, which is negligible
compared to FAISS and LLM calls.

\paragraph{BatchGraphDenseRetriever.}
For graph-dense retrieval, we cache query entities and expand the graph once per
query up to the maximum hop, entity-frequency, and candidate-count limits used
by the pool. Each graph-dense hyperparameter setting is then evaluated by
filtering this shared expansion context and reranking its candidate chunks with
the requested dense backbone. When both MPNet and E5 variants are present, the
query embeddings are loaded as a dictionary keyed by backbone and the same graph
candidate set is reranked with the matching chunk-embedding matrix.

\subsection{LLM prompting and answer selection}
\label{app:prompts}

We use the same prompt templates across methods to ensure fair comparison.

\paragraph{Answer prompt (Answer LLM).}
Given a question and the retrieved passages, the answerer is instructed to be
\emph{evidence-first} and to output a short final answer inside an
\texttt{<answer>...</answer>} tag, along with a brief structured explanation
indicating whether the answer is supported by the passages (\texttt{mode=evidence})
or a fallback guess (\texttt{mode=guess}). This structured output is used by the
judge during selection.

The exact system \& user prompts that we are using are the following.

\begin{promptbox}{Answer Prompt (System)}
You are a helpful question-answering assistant.
Your objective is to produce the best SHORT answer.

Core rules:
1) Evidence-first: If the passages contain explicit evidence that entails the answer, use it. Prefer statements that are specific and unambiguous; favor answers supported by multiple passages.
2) No evidence -> best-guess: If the passages are irrelevant, too vague, or do not entail an answer, give your best-guess from your general knowledge, but mark mode='guess'.
3) Never contradict the passages: If any passage clearly contradicts your prior knowledge, trust the passages unless they are clearly off-topic (irrelevant to the question). Do not invent unsupported details.
4) Be concise: The answer must be a single word, name, date, number, or very short phrase.
5) Always put the final answer inside <answer>...</answer> tags.

Conflict handling:
- If passages disagree, pick the answer with the strongest explicit support (more passages, clearer wording). 
- If the evidence is ambiguous, output your best guess but mark the mode as 'guess' and explicitly mention this in your explanation.
- If multi-hop reasoning is needed, combine facts across passages explicitly.

Output format (exactly these two blocks, in this order):
<explanation mode="evidence|guess" used_passages="[comma-separated indices]">
Brief 1-2 sentence justification referencing passage indices (e.g., 'P2 states X; P4 confirms Y').
</explanation>
<answer>YOUR SHORT ANSWER</answer></final>
\end{promptbox}

\begin{promptbox}{Answer Prompt (User)}
Here are a few examples of valid outputs:
Example question: In what year was Google founded?
Example answer (evidence mode):
<explanation mode="evidence" used_passages="[1,3]">
P1 names the founder; P3 gives the company's founding year.
</explanation>
<answer>1998</answer></final>

Example question: Is Saturn larger than Jupiter?
Example answer (guess mode, no usable evidence in passages):
<explanation mode="guess" used_passages="[]">
No passage compares the sizes; providing best-guess from general knowledge.
</explanation>
<answer>No</answer></final>

Example question: Is Kyoto the capital of Japan?
Example answer (evidence mode):
<explanation mode="evidence" used_passages="[0]">
P0 states that Tokyo is Japan's capital city, so the answer is no.
</explanation>
<answer>No</answer></final>

Example question: Who has scored the most points in NBA history?
Example answer (evidence mode):
<explanation mode="evidence" used_passages="[1]">
P1 explicitly states that LeBron James is the NBA's all-time leading scorer.
</explanation>
<answer>LeBron James</answer></final>

Example question: Which is larger, Saturn or Jupiter?
Example answer (guess mode, no usable evidence in passages):
<explanation mode="guess" used_passages="[]">
No passage compares sizes; providing best-guess from general knowledge.
</explanation>
<answer>Jupiter</answer></final>
========
Now answer the following question.
Question:
{QUESTION}

Context Passages:
Passage 0: {PASSAGE_0}

Passage 1: {PASSAGE_1}

... (repeat Passage blocks as needed) ...

========
Your task:
- Read the passages and determine whether they explicitly support an answer.
- If yes, answer in EVIDENCE mode and cite the passage indices you used.
- If no, answer in GUESS mode (best-guess from your knowledge).
- Do NOT repeat the question. Keep the answer minimal. No extra text outside the required blocks.
- ALWAYS conclude your answer with a </final> tag and write nothing after it!
======
\end{promptbox}

\paragraph{Selection prompt (Selector LLM).}
When we evaluate multiple retrievers per query ($\ell>1$), we run the answerer
once per retriever and then prompt a judge model with only the question and the
candidate answers (including their structured explanations). The judge selects
the single most trustworthy candidate and outputs the chosen answer in a
\texttt{<judge>...</judge>} tag.

The exact system and user prompts that we used are the following.

\begin{promptbox}{Selector Prompt (System)}
You are a question-answering judge.
- You will be given a question and several candidate answers.
- Each answer includes an explanation and a mode (evidence or guess).
- Your task is to select the SINGLE answer that is most trustworthy.
- Prefer answers in evidence mode with clear, well-structured explanations.
- If no evidence-mode answers are persuasive, you may choose a guess-mode answer, but only if the explanation is reasonable and consistent with the question.
- Never invent a new answer: always pick one of the candidates.
- Provide a short justification (2-5 sentences) comparing candidates, referencing their mode and the quality of their explanations.

STRICT FORMAT:
- End with a single line ONLY: <judge>ACTUAL_ANSWER</judge></final>
- IMPORTANT: Replace ACTUAL_ANSWER with the real chosen answer.
- Do NOT output the word ACTUAL_ANSWER literally.
- Examples of correct final lines:
    <judge>1998</judge></final>
    <judge>Eiffel Tower</judge></final>
\end{promptbox}

\begin{promptbox}{Selector Prompt (User)}
Question:
{QUESTION}

Candidate Answers:
======
Answer 0:
{ANSWER_0}
------
Answer 1:
{ANSWER_1}
------
Answer 2:
{ANSWER_2}
------
... (repeat Answer blocks as needed) ...
======
Your task:
- Compare the answers using their explanations and mode.
- Write a 2-5 sentence justification.
- END with one line ONLY: <judge>ACTUAL_ANSWER</judge></final>
- IMPORTANT: Replace ACTUAL_ANSWER with the real chosen answer.
\end{promptbox}

\subsection{Evaluation metrics}
\label{app:metrics}

\paragraph{Retrieval Metrics} We evaluate retrieval performance against a set of \textbf{reference supporting documents} that are provided for each question in the datasets. Since retrieval operates at the chunk level, a retrieved chunk is considered a ``hit'' if its document ID matches one of the \textbf{reference document} identifiers. We report document-level performance using the following metrics at a fixed budget of $n$ retrieved items:

\begin{itemize}
    \item \textbf{Support Recall@$k$}: The fraction of reference documents successfully identified within the top $k$ retrieved results.
    \item \textbf{Support F1@$k$}: The harmonic mean of precision and recall at budget $k$. This provides a balanced measure that accounts for both the accuracy (precision) of the retrieved chunks and the completeness (recall) of the retrieved support set.
\end{itemize}

\paragraph{End-to-end QA.} 
We report Exact Match (EM) scores for the final selected answer. A prediction is considered correct if, after standard normalization (lowercasing, stripping
punctuation and articles), it exactly matches any of the reference ground-truth answers. To characterize the efficiency-accuracy trade-offs of our approach, we further measure the total tokens generated and the end-to-end wall-clock latency of the routed pipeline.

\paragraph{Retrieve-more controls.}
To separate portfolio complementarity from simply increasing context length, we
also evaluate single-retriever controls that retrieve a larger number of
documents under the same answer-generation prompt format. These controls use
one retriever configuration rather than a portfolio and therefore do not benefit
from best-of-$k$ coverage across complementary retrieval behaviors.

\subsection{Router model and training}
\label{app:router}

We train a query router that maps each question $q$ to a ranking over the $k$ retrievers in a learned portfolio.

\paragraph{Inputs.}
Each question $q$ is represented by:
(i) its raw text (tokenized for a T5 encoder),
(ii) a dense MPNet query embedding
$e_{\text{mp}}(q)\in\mathbb{R}^{d_{\text{mp}}}$, and
(iii) a dense E5 query embedding
$e_{\text{e5}}(q)\in\mathbb{R}^{d_{\text{e5}}}$.
The dense query embeddings are loaded from the same per-backbone caches used by
retrieval. The router also carries binary masks for the dense modalities so a
missing embedding can be zeroed after projection.

\paragraph{Architecture.}
Let $h_{\text{t5}}(q)\in\mathbb{R}^{d_{\text{t5}}}$ denote the pooled
\textsc{Flan-T5-large} encoder output.
We pool by averaging the encoder hidden states across tokens using the attention mask.
The T5 encoder is frozen throughout router training.

We project the text representation and both dense embeddings to a common
dimension $d=256$:
\[
\tilde{h}(q) = \phi_h\!\left(W_h\,h_{\text{t5}}(q)\right),\qquad
\tilde{m}(q) = m_{\text{mp}}\phi_{\text{mp}}\!\left(W_{\text{mp}}\,e_{\text{mp}}(q)\right),\qquad
\tilde{e}(q) = m_{\text{e5}}\phi_{\text{e5}}\!\left(W_{\text{e5}}\,e_{\text{e5}}(q)\right),
\]
where each $\phi(\cdot)$ is a linear layer followed by ReLU, dropout, and
LayerNorm, and $m_{\text{mp}},m_{\text{e5}}\in\{0,1\}$ are modality masks.

We fuse the three projected vectors with a small MLP and LayerNorm:
\[
z(q)=\mathrm{LayerNorm}\!\Big(\mathrm{MLP}\big([\tilde{h}(q);\tilde{m}(q);\tilde{e}(q)]\big)\Big)\in\mathbb{R}^{256},
\qquad
\hat z(q)=\frac{z(q)}{\|z(q)\|_2}.
\]

The router learns one trainable prototype vector per retriever in the portfolio, $\{v(r)\in\mathbb{R}^{256}\}_{r=1}^k$,
and uses cosine similarity scores:
\[
\textnormal{sim}(q,r)=\cos\!\big(\hat z(q), \hat v(r)\big)
\quad\text{where}\quad
\hat v(r)=\frac{v(r)}{\|v(r)\|_2}.
\]
At inference time, the router ranks retrievers by $\textnormal{sim}(q,r)$ and
executes either the argmax retriever or the top-$\ell$ retrievers, depending on
the evaluation budget.

\paragraph{Training supervision from retrieval recall.}
For each training question $q$, we precompute a vector of retrieval scores
$\mathrm{Rec}@K(q)\in\mathbb{R}^k$, one value per portfolio member, where $K$ is the number of documents retrieved for evaluation.
We define the \emph{winner set} as the set of retrievers achieving the maximum recall on that instance:
\[
W(q)=\Big\{r:\mathrm{Rec}@K(q,r)=\max_{r'}\mathrm{Rec}@K(q,r')\Big\}.
\]
Questions for which $\max_r \mathrm{Rec}@K(q,r)=0$ provide no routing signal; we exclude them from training/evaluation.
To avoid unstable supervision from very large ties, we additionally ignore questions with $|W(q)|>T$ (default $T=3$).

\paragraph{Contrastive argmax objective (multi-positive).}
Let $p(r\mid q)=\mathrm{softmax}_r\big(s(q,r)\big)$ be the distribution induced by the cosine scores.
We minimize a multi-positive contrastive argmax loss that concentrates probability mass on the winner set:
\[
\mathcal{L}(q)
=
-\frac{1}{|W(q)|}\sum_{r\in W(q)}\log p(r\mid q).
\]
This treats all tied-best retrievers as positives and encourages the router to assign them high likelihood.

\paragraph{Balanced batching.}
To reduce dominance of frequent single-winner cases, we bin training questions by their winner structure
(e.g., ``$R_i$-unique'' when retriever $i$ is the unique winner, and optional tie bins such as tie2/tie3),
and sample batches in a round-robin manner across non-empty bins.

\paragraph{Optimization and schedule.}
We keep the T5 encoder frozen and optimize only the dense projection heads,
fusion MLP, and retriever prototype vectors.
Checkpoints store only non-T5 weights.
We optimize with AdamW using a single learning rate for all trainable
parameters.
Model selection is performed by highest dev recall@K under $\arg\max_r s(q,r)$ routing.

\paragraph{Default hyperparameters.}
Unless otherwise stated, we use:
\begin{itemize}
    \item Epochs: $10$.
    \item Batch size: $64$; gradient accumulation: $1$ step; gradient clipping: $1.0$.
    \item Learning rate: $3\times 10^{-4}$ for all trainable non-T5 parameters.
    \item Max input length: $256$ tokens.
    \item Max tie size: $T=3$ (ignore questions $q$ with $|W(q)|>T$).
\end{itemize}

\subsection{Vendi-RAG baseline details}
\label{app:vendirag}

For the Vendi-RAG baseline \citep{rezaei2025vendi}, our implementation follows the
one described in \citet{rezaei2025vendi}. We implement a $T$-step
adaptive loop over the Vendi trade-off parameter $s$:
starting from an initial $s_0 = 0.8$, we repeatedly (i) retrieve with Vendi at $s_t$,
(ii) answer with the same answer LLM used elsewhere, (iii) score the answer with
a judge LLM, and (iv) update $s$ for the next iteration. Our update rule maps the
judge score (clamped to $[0,10]$) to $s_{t+1}=1-\text{score}/10$, i.e., lower
judged answer quality increases $s$ and shifts subsequent retrieval toward more
diverse evidence.
The final answer is the step with the best judge score. 

\section{Detailed proofs and explanations of \cref{sec:theory}}
\label{app:greedy_portfolio_proofs}
We give a detailed analysis of \cref{alg:efficient_greedy}, and, in particular, we prove \cref{thm:greedy-analysis}. The running time is immediate from the explanations and the pseudocode in \cref{sec:theory}. We thus focus on the approximation guarantee and the sample complexity $N$, which follows from that the function is monotone and submodular, and a simple uniform convergence bound (via Hoeffding's inequality and a union bound).

Recall that a set function $f:2^V\to\mathbb{R}$ is \emph{normalized} if $f(\emptyset)=0$. 
It is \emph{monotone} if $f(A)\le f(B)$ for all $A\subseteq B\subseteq V$. 
It is \emph{submodular} if it satisfies the diminishing returns property, i.e., for all 
$A\subseteq B\subseteq V$ and $x\in V\setminus B$,
\[
f(A\cup\{x\})-f(A)\ge f(B\cup\{x\})-f(B).
\]
Further note that a normalized, monotone set function is non-negative.

\subsection{Submodularity of the best-of-$k$ objective}

For a fixed query $q$, define
$f_q(S):=\max_{r\in S}s(q,r)$ with $f_q(\emptyset)=0$.
Recall $s(q,r)\in[0,1]$. 

\begin{lemma}[Pointwise monotone submodularity]
\label{lem:pointwise_submodular}
For any fixed $q$, the set function $f_q:2^{\mathcal R}\to[0,1]$ is normalized,
monotone, and submodular.
\end{lemma}

\begin{proof}
Normalization holds by definition. Monotonicity follows since a maximum cannot
decrease when adding elements.

For submodularity, let $A\subseteq B\subseteq\mathcal R$ and $r\in\mathcal R\setminus B$.
Let $v_A=\max_{r'\in A}s(q,r')$ and $v_B=\max_{r'\in B}s(q,r')$, so $v_A\le v_B$.
Then
\[
f_q(A\cup\{r\})-f_q(A)=\max(0,\,s(q,r)-v_A),
\quad
f_q(B\cup\{r\})-f_q(B)=\max(0,\,s(q,r)-v_B).
\]
Since $v_B\ge v_A$, we have $\max(0,s(q,r)-v_A)\ge \max(0,s(q,r)-v_B)$, proving
diminishing returns.
\end{proof}

Given $Q=\{q_1,\dots,q_N\}\sim\mathcal D^N$, define the empirical objective
\[
\widehat F_Q(S):=\frac1N\sum_{i=1}^N f_{q_i}(S)
=\frac1N\sum_{i=1}^N\max_{r\in S}s(q_i,r).
\]
\begin{lemma}[Population monotone submodularity]
\label{lem:population_submodular}
The objectives $F(S)=\mathbb E_{q\sim\mathcal D}[f_q(S)]$ and $\widehat F_Q(S)$ are normalized, monotone,
and submodular.
\end{lemma}

\begin{proof}
By Lemma~\ref{lem:pointwise_submodular}, each $f_q(\cdot)$ is normalized, monotone,
and submodular. Nonnegative linear combinations (and thus expectations) preserve
these properties. 
\end{proof}

\subsection{Uniform convergence from finitely many query samples}

\begin{lemma}[Uniform additive concentration over size-$k$ portfolios]
\label{lem:uniform_convergence}
Fix $\epsilon\in(0,1)$ and $\delta\in(0,1)$. Let $m:=|\mathcal R|$ and
\[
M:=\sum_{j=0}^k {m\choose j}\;\;\le\;\;\Bigl(\frac{e m}{k}\Bigr)^k\quad (1\le k\le m).
\]
If
\[
N \;\ge\; \frac{1}{2\epsilon^2}\log\frac{2M}{\delta},
\]
then with probability at least $1-\delta$ (over $Q\sim\mathcal D^N$),
\[
\sup_{S\subseteq\mathcal R:\,|S|\le k}\left|\widehat F_Q(S)-F(S)\right|\;\le\;\epsilon.
\]
In particular it suffices that
$N=\mathcal O\!\left(\frac{k\log m+\log(1/\delta)}{\epsilon^2}\right)$.
\end{lemma}

\begin{proof}
Fix any $S$ with $|S|\le k$. The variables $X_i:=f_{q_i}(S)\in[0,1]$ are i.i.d.\ with
$\mathbb E[X_i]=F(S)$. By Hoeffding's inequality,
\[
\Pr\left[\left|\widehat F_Q(S)-F(S)\right|>\epsilon\right]\le 2e^{-2N\epsilon^2}.
\]
Union bound over the at most $M$ candidate portfolios gives
\[
\Pr\left[\exists S:\,|S|\le k \text{ s.t. } \left|\widehat F_Q(S)-F(S)\right|>\epsilon\right]
\le 2M e^{-2N\epsilon^2}.
\]
Choosing $N\ge \frac{1}{2\epsilon^2}\log\frac{2M}{\delta}$ makes the RHS at most $\delta$.
The simplification uses $M\le (em/k)^k$.
\end{proof}

\subsection{End-to-end population guarantee}
The following theorem implies the stated guarantee of \cref{thm:greedy-analysis}.
\begin{theorem}[Near-optimality of Algorithm~\ref{alg:efficient_greedy}]
\label{thm:near_optimality}
Let $m:=|\mathcal R|$. Fix $k\in\mathbb N$, $\epsilon\in(0,1)$, and $\delta\in(0,1)$.
Draw $Q\sim\mathcal D^N$ with
$N=\mathcal O\!\left(\frac{k\log m+\log(1/\delta)}{\epsilon^2}\right)$
(as in Lemma~\ref{lem:uniform_convergence}).
Let $S_g$ be the output of Algorithm~\ref{alg:efficient_greedy} (greedy on $\widehat F_Q$).
Then with probability at least $1-\delta$,
\[
F(S_g)\;\ge\;\Bigl(1-\frac1e\Bigr)\max_{|S|\le k}F(S)\;-\;\mathcal O(\epsilon).
\]
More explicitly, if Lemma~\ref{lem:uniform_convergence} holds with deviation at most
$\eta$, then
\[
F(S_g)\;\ge\;\Bigl(1-\frac1e\Bigr)\max_{|S|\le k}F(S)\;-\;\Bigl(2-\frac1e\Bigr)\eta.
\]
\end{theorem}

\begin{proof}
Condition on the event from Lemma~\ref{lem:uniform_convergence} that
\[
\sup_{|S|\le k}\bigl|\widehat F_Q(S)-F(S)\bigr|\le \eta.
\]
Let $S^\star\in\arg\max_{|S|\le k}F(S)$ and
$\widehat S^\star\in\arg\max_{|S|\le k}\widehat F_Q(S)$.

\paragraph{Step 1: Algorithm~\ref{alg:efficient_greedy} is greedy on $\widehat F_Q$.}
Fix an intermediate greedy set $S$, and define $V[q_i]=\max_{r'\in S}s(q_i,r')$ as in
Algorithm~\ref{alg:efficient_greedy}. Then for any candidate $r\notin S$,
\begin{align*}
\widehat F_Q(S\cup\{r\})-\widehat F_Q(S)
&=\frac1N\sum_{i=1}^N\Big(\max_{r'\in S\cup\{r\}}s(q_i,r')-\max_{r'\in S}s(q_i,r')\Big)\\
&=\frac1N\sum_{i=1}^N\big(\max(V[q_i],s(q_i,r)) - V[q_i]\big)\\
&=\frac1N\sum_{i=1}^N\max\bigl(0,\,s(q_i,r)-V[q_i]\bigr).
\end{align*}
Thus maximizing $\sum_{i=1}^N\max(0,s(q_i,r)-V[q_i])$ (as done in the algorithm) is
equivalent to maximizing the empirical marginal gain
$\widehat F_Q(S\cup\{r\})-\widehat F_Q(S)$.

\paragraph{Step 2: Greedy achieves $(1-1/e)$ on $\widehat F_Q$.}
By Lemma~\ref{lem:population_submodular}, $\widehat F_Q(\cdot)$ is normalized, monotone,
and submodular. Therefore, by the classical result of Nemhauser--Wolsey--Fisher
\cite{nemhauser1978mp}, the greedy algorithm returns $S_g$ with $|S_g|=k$ such that
\[
\widehat F_Q(S_g)\;\ge\;\Bigl(1-\frac1e\Bigr)\widehat F_Q(\widehat S^\star)
\;\ge\;\Bigl(1-\frac1e\Bigr)\widehat F_Q(S^\star).
\]

\paragraph{Step 3: Transfer from empirical to population.}
Using the uniform deviation bound,
\[
F(S_g)\ge \widehat F_Q(S_g)-\eta
\ge \Bigl(1-\frac1e\Bigr)\widehat F_Q(S^\star)-\eta
\ge \Bigl(1-\frac1e\Bigr)\bigl(F(S^\star)-\eta\bigr)-\eta,
\]
which simplifies to
\[
F(S_g)\ge \Bigl(1-\frac1e\Bigr)F(S^\star)-\Bigl(2-\frac1e\Bigr)\eta
=\Bigl(1-\frac1e\Bigr)\max_{|S|\le k}F(S)-\Bigl(2-\frac1e\Bigr)\eta.
\]
Finally, choosing $N$ as in Lemma~\ref{lem:uniform_convergence} makes $\eta=\Theta(\epsilon)$,
yielding the stated $\mathcal O(\epsilon)$ additive term.
\end{proof}


\end{document}